\documentclass{article} 

\PassOptionsToPackage{table,dvipsnames}{xcolor}
\usepackage[T1]{fontenc}
\usepackage{iclr2027_conference,times}

\usepackage{amsmath,amsfonts,bm}

\def\eqref#1{equation~\ref{#1}}

\def\1{\bm{1}}

\DeclareMathAlphabet{\mathsfit}{\encodingdefault}{\sfdefault}{m}{sl}
\SetMathAlphabet{\mathsfit}{bold}{\encodingdefault}{\sfdefault}{bx}{n}

\newcommand{\Ls}{\mathcal{L}}

\usepackage{hyperref}
\usepackage{url}
\usepackage{booktabs}
\usepackage{amssymb}
\usepackage{xcolor}
\usepackage{multirow}
\usepackage{graphicx}
\usepackage{wrapfig}
\usepackage{needspace}
\usepackage{makecell}
\usepackage{amsthm}
\usepackage{float}
\usepackage[ruled,vlined]{algorithm2e}

\definecolor{RLPanel}{RGB}{228,238,250}
\definecolor{MuonPanel}{RGB}{241,236,251}
\definecolor{SFTPanel}{RGB}{231,244,235}
\definecolor{EffTint}{HTML}{EEF4FC}

\definecolor{ProbeC}{HTML}{2A78D6}
\definecolor{PredC}{HTML}{C98A06}
\definecolor{CorrC}{HTML}{2E9E6B}
\definecolor{CardTint}{HTML}{F5F8FC}

\newcommand{\yes}{\textcolor{CorrC}{$\checkmark$}}
\newcommand{\no}{\textcolor{BrickRed}{$\times$}}
\newcommand{\na}{\textcolor{gray}{--}}

\newcommand{\mbadge}[2]{%
  \colorbox{ProbeC!#1}{\makebox[2.2em]{#2$\times$}}%
}

\definecolor{OPDPanel}{RGB}{252,239,221}
\definecolor{PrePOGreen}{RGB}{0,105,65}

\newcommand{\good}[1]{%
  \textcolor{PrePOGreen}{\normalsize $^{+#1}$}%
}
\newcommand{\bad}[1]{%
  \textcolor{BrickRed}{\normalsize $^{-#1}$}%
}

\newcommand{\tableobjective}[1]{%
  \textcolor{black}{\pmb{\textit{#1}}}%
}

\renewcommand{\arraystretch}{1.15}

\newsavebox{\fitwidthbox}

\newcommand{\fitwidth}[1]{%
  \savebox{\fitwidthbox}{#1}%
  \ifdim\wd\fitwidthbox>\textwidth
    \resizebox{\textwidth}{!}{\usebox{\fitwidthbox}}%
  \else
    \usebox{\fitwidthbox}%
  \fi
}

\newtheorem{assumption}{Assumption}
\newtheorem{lemma}{Lemma}
\newtheorem{theorem}{Theorem}
\newtheorem{corollary}{Corollary}
\newtheorem{proposition}{Proposition}

\title{
\includegraphics[height=1.5em]{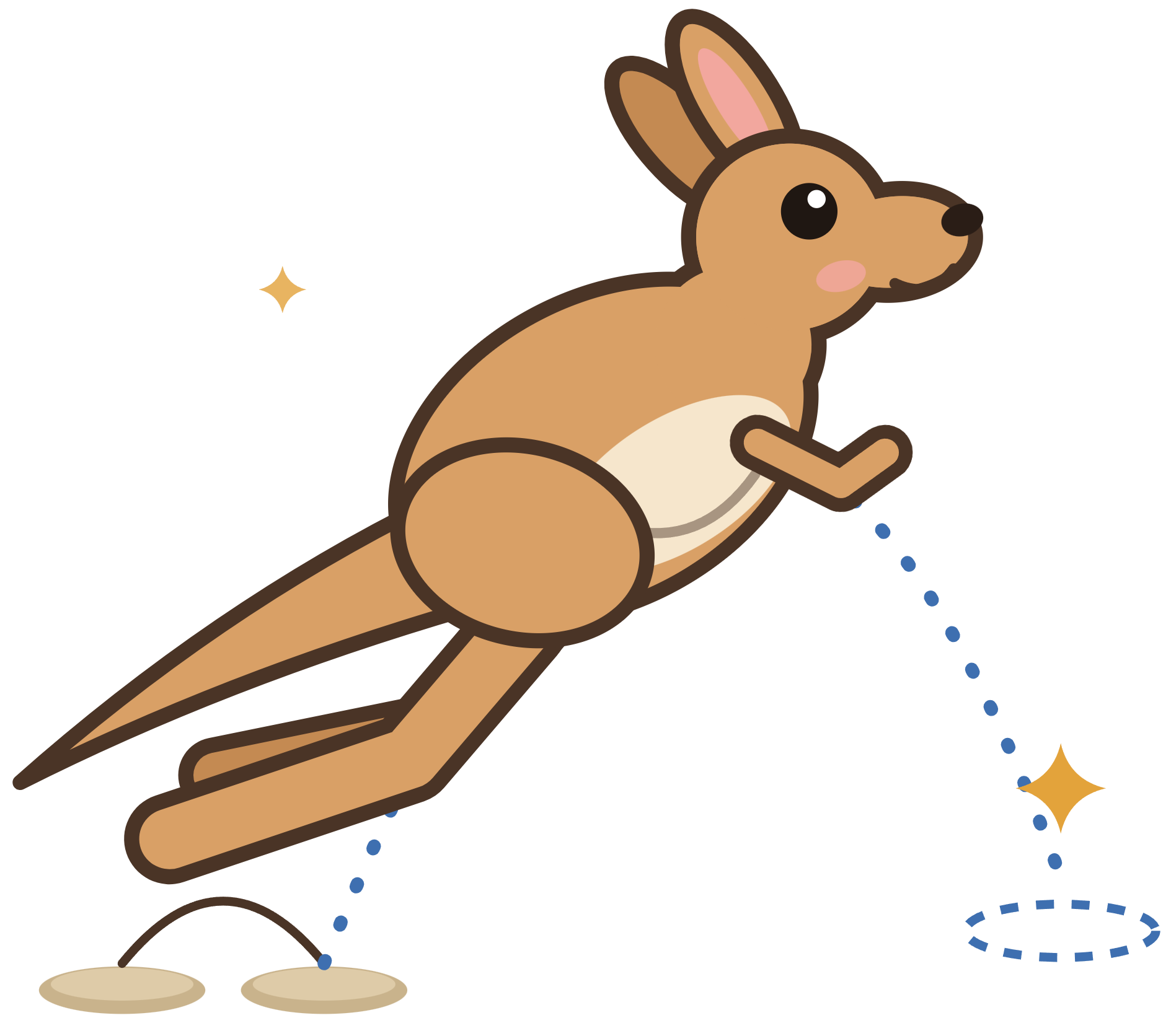}\hspace{0.05em} Skip What You Can Predict: Predictive Repositioning for Policy Optimization for Efficient LLM Training
}

\author{Ismam Nur Swapnil\textsuperscript{1,2} \quad
Aranya Saha\textsuperscript{3} \quad
Tasneea Zahra\textsuperscript{4} \quad
Tanvir Ahmed Khan\textsuperscript{5}
\\
\textbf{Mohammad Ariful Haque}\textsuperscript{\textbf{2}} \quad
\textbf{Ser-Nam Lim}\textsuperscript{\textbf{6}}
\\[4pt]
\textsuperscript{1}Advanced Chemical Industries Ltd. \quad
\textsuperscript{2}Bangladesh University of Engineering and Technology
\\
\textsuperscript{3}University of Maryland, College Park \quad
\textsuperscript{4}Khulna University of Engineering and Technology
\\
\textsuperscript{5}Illinois Institute of Technology \quad
\textsuperscript{6}University of Central Florida
}

\iclrfinalcopy

\begin{document}

\maketitle

\begin{abstract}
Reinforcement learning with verifiable rewards (RLVR) can improve the reasoning ability of large language models, but repeatedly updating a policy on the same rollout batch is expensive: every additional update requires another backward pass, and multi-step methods pay for all intermediate optimization steps. We introduce Predictive Repositioning for Policy Optimization (PrePO), which uses two observed optimizer transitions to estimate a farther point along the same-batch optimization trajectory, moves partway toward that point, evaluates the original objective there, and applies a corrective update. This gives PrePO a fixed active cost that does not grow with the virtual optimization depth. Our analysis gives finite-horizon error bounds for AdamW and Muon and shows that sufficiently accurate endpoint estimates preserve the usual descent and convergence behavior of smooth gradient descent. In RLVR experiments, PrePO reaches matched performance targets in fewer optimization steps and lower wall-clock time than the corresponding baselines. We further evaluate the same update mechanism in supervised fine-tuning on a different dataset, showing that its use is not restricted to the original RLVR setting. Together, these results suggest that PrePO provides a practical mechanism for approximating repeated same-batch optimization while avoiding the cost of explicitly executing every intermediate update.
\end{abstract}

\begin{figure}[H]
\centering
\includegraphics[width=\textwidth]{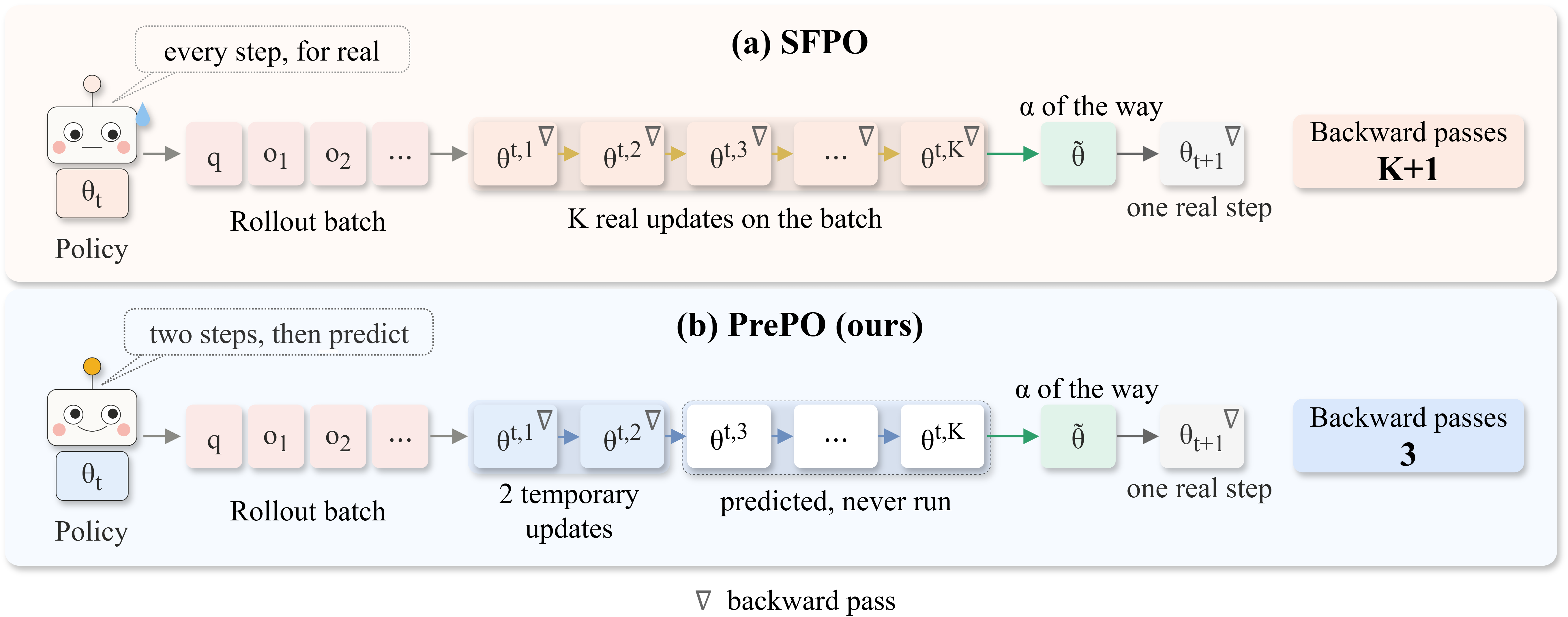}
\setlength{\abovecaptionskip}{2pt}
\caption{SFPO and PrePO on the same rollout batch. (a) SFPO runs all $K$ updates for real before it
repositions and corrects, so its cost grows with $K$. (b) PrePO runs two temporary updates, predicts the
rest of the path from how the motion is changing, moves a fraction $\alpha$ toward the predicted endpoint,
and takes one real step there, three backward passes for any $K$.}
\label{fig:prepo-overview}
\end{figure}

\section{Introduction}
\label{sec:intro}

Reinforcement learning with verifiable rewards (RLVR) improves an LLM by
repeatedly updating a policy from sampled responses and their verified rewards.
The rollout, reward, and advantage computations are already costly, and every
policy update adds another backward pass. In practice, one rollout batch is often reused for several updates, as in the multiple epochs of
PPO \citep{Schulman2017PPO}, so the policy follows a local trajectory before the
next batch is collected. Reaching a target accuracy therefore depends not only on
the quality of the learning signal, but also on how many updates are executed on
each batch.

Several RL optimizers use more than one update from the same batch. GRPO learns
from group-relative rewards without a separate critic \citep{Shao2024DeepSeekMath},
while SFPO executes several policy updates on one rollout batch before applying a
correction \citep{Wang2026SFPO}. Classical ideas such as extragradient, vector
extrapolation, and Lookahead also use trial points, extrapolated iterates, or fast inner
updates \citep{Korpelevich1976,Sidi1986,Zhang2019Lookahead}. These methods can move farther along a local trajectory, but
the intermediate updates are real optimizer steps, so their backward-pass cost
grows with the number of steps. This raises a basic question for RLVR: \textbf{can the
trajectory information be retained without executing every intermediate policy
update?}

We address this question with Predictive Repositioning for Policy Optimization
(PrePO, Figure~\ref{fig:prepo-overview}). PrePO is developed first for RLVR: it uses two temporary policy transitions
to estimate where a farther update would land, moves partway toward that estimate,
and lets the original policy objective determine the committed update. The data,
reward construction, loss, and optimizer remain unchanged; only the execution of
the update sequence changes. The active path uses three optimizer updates,
independent of the virtual depth $K$.
After establishing the RLVR result, we test whether the same schedule transfers to
supervised fine-tuning (SFT) and on-policy distillation (OPD), and we evaluate Muon \citep{Jordan2024Muon} as an
alternative optimizer.

\section{Related Work}
\label{sec:related}

Our primary setting is reinforcement learning with verifiable rewards (RLVR), where policy-gradient
methods improve a model from sampled responses and their rewards. GRPO learns from group-relative
rewards without a separate critic \citep{Shao2024DeepSeekMath, Mroueh2025GRPODynamics,
Mroueh2025RevisitingGRPO}; DAPO changes clipping and dynamic sampling for large-scale training
\citep{Yu2025DAPO}; GSPO moves importance-ratio computation and clipping to the sequence level
\citep{Zheng2025GSPO}; CAPO uses curvature information to identify unstable token updates
\citep{Melo2026CAPO}; and VAPO develops a value-based augmented PPO framework for long-chain
reasoning \citep{Yue2025VAPO}. Other work reduces rollout or systems cost through selective sampling,
adaptive allocation, distributed orchestration, or speculative decoding
\citep{Zheng2025SelectiveRollouts, Wu2026MovingEdge, Jiang2026ProgressiveRollout, Xu2026PODS,
Sheng2025HybridFlow, He2025RhymeRL, Kim2026EfficientRollout}. These methods modify the learning
signal, data budget, or training pipeline. PrePO starts after the rollout batch is fixed: it takes two
ordinary updates as temporary probes, predicts how the same-batch motion would continue, moves partway
toward that predicted point, and keeps one corrective update.

Our transfer analysis also covers supervised fine-tuning \citep{Wei2022FLAN, Ouyang2022InstructGPT} and
on-policy distillation \citep{Hinton2015Distilling, Gu2024MiniLLM, Agarwal2024GKD}. There too PrePO leaves the
data, parameterization, and objective fixed and changes only how the objective is executed on a batch.

\textbf{Extrapolation in optimization.} Table~\ref{tab:extrapolation} sorts classical extrapolation
methods by how far ahead they look and how they get there. Extragradient, mirror-prox, and optimistic methods
\citep{Korpelevich1976,Nemirovski2004,Popov1980,Rakhlin2013Optimistic,Daskalakis2018Optimism,Gidel2019VI,Mertikopoulos2019,Mokhtari2020Unified},
momentum \citep{Polyak1964,Nesterov1983}, and gradient extrapolation
\citep{LanZhou2018RGEM,Lin2020Extrapolation} look one step ahead. Vector extrapolation and Anderson
acceleration look to where an iteration finally settles, using a stored history of iterates
\citep{Richardson1911,Aitken1927,CabayJackson1976,Sidi1986,Anderson1965,WalkerNi2011,Scieur2016RNA}.
Lookahead, policy mirror descent with lookahead, and SFPO reach a point $K$ steps ahead by taking every
step \citep{Zhang2019Lookahead,Protopapas2024Lookahead,Wang2026SFPO}. Weight prediction looks a small fixed number of steps $s$
ahead by assuming the current update repeats unchanged \citep{Chen2018SpecTrain,Guan2023XGrad}, like
predicting where a car will be from its speed alone.

PrePO predicts from the speed \emph{and how it is changing}. Two probe steps on the batch show, for each
coordinate, whether the optimizer's motion is holding its pace or slowing down, and PrePO carries that
trend $K$ steps forward for three backward passes and checks it with a fresh corrective gradient.
Repeating the current update, like a larger learning rate, stretches every coordinate by the same amount,
including the ones going back and forth, so on sharp coordinates its step grows
(Proposition~\ref{prop:not_larger_step}); with $R=I$ PrePO reduces to this constant-step prediction.
Other work predicts the training run rather than the batch, carrying weights forward along a nearly
straight line through saved checkpoints \citep{Zheng2024ExPO,Cai2026AlphaRL,Wang2026LinearDynamics}.
PrePO works inside a single batch, where the trend is measured afresh each time and every predicted point
is followed by a real gradient step, so the two act at different scales.

\begin{table}[t]
\caption{Where each extrapolation family aims and how it gets there. A fixed point is where a repeated
iteration stops moving, which a one-update RLVR batch never reaches; $s$ is a small fixed horizon such as a pipeline delay.}
\label{tab:extrapolation}
\centering\small
\setlength{\tabcolsep}{3pt}
\begin{tabular}{>{\raggedright\arraybackslash}p{0.215\linewidth}>{\raggedright\arraybackslash}p{0.10\linewidth}>{\raggedright\arraybackslash}p{0.345\linewidth}cc>{\columncolor{EffTint}\raggedright\arraybackslash}p{0.14\linewidth}}
\toprule
\rowcolor{gray!15}
Family & Aims at & How the lookahead point is obtained & \makecell{Tracks\\motion\\change} & \makecell{Cost\\fixed\\in $K$} & Backward passes \\
\midrule
Extragradient, optimistic, momentum & 1 step & a trial step, the previous gradient, or past iterates & \na & \yes & 1--2 per step \\
Vector extrapolation, Anderson & fixed point & least-squares fit over a stored history of iterates & \yes & \no & 1 per iterate, to convergence \\
Weight prediction & $s$ steps & repeat the current update unchanged & \no & \yes & 1 per step \\
Lookahead, SFPO & $K$ steps & execute all $K$ updates & \yes & \no & $K$ or $K{+}1$ \\
Checkpoint extrapolation & later run & extend a line through saved checkpoints & \no & \yes & none, no correction \\
\rowcolor{RLPanel}
\textbf{PrePO (ours)} & $K$ steps & fit the motion trend from two probes, project it $K$ steps & \yes & \yes & \textbf{\textcolor{ProbeC}{3 for any $K$}} \\
\bottomrule
\end{tabular}
\end{table}

\section{Background: RL Updates and Their Extensions}
\label{sec:background}

In RLVR, the current policy samples a rollout batch
$B_{\mathrm{RL}}$ containing prompts, responses, rewards, and the quantities needed by the policy
objective. Once this batch is fixed, GRPO or another policy optimizer evaluates the same objective at
successive parameter states. SFT and OPD expose the same update interface with different batches and
losses: $B_{\mathrm{SFT}}$ is a supervised minibatch, while $B_{\mathrm{OPD}}$ contains the inputs and
teacher targets used by the on-policy distillation loss.

Let $\mathcal{T}\in\{\mathrm{RL},\mathrm{SFT},\mathrm{OPD}\}$ index the training setting, let
$\Ls_{\mathcal{T}}(\theta;B_{\mathcal{T}})$ be its objective, and let $s^j$ be the state of a possibly
stateful optimizer. At update $j$, the objective gradient and optimizer transition are
\begin{equation}
g_{\mathcal{T}}^j = \nabla_\theta \Ls_{\mathcal{T}}(\theta^j; B_{\mathcal{T}}), \qquad
(\theta^{j+1},s^{j+1}) = \mathcal{U}_{\mathcal{T}}(\theta^j,s^j;g_{\mathcal{T}}^j),
\label{eq:grad}
\end{equation}
where $\mathcal{U}_{\mathcal{T}}$ includes the chosen optimizer and its state. The observed parameter
displacement is $\Delta^j=\theta^{j+1}-\theta^j$. Reusing the same batch for several updates produces
the local trajectory $\theta^0\rightarrow\theta^1\rightarrow\cdots\rightarrow\theta^K$. Under
plain gradient descent, this trajectory satisfies
\begin{equation}
\theta^K = \theta^0 - \eta\sum_{j=0}^{K-1} g_{\mathcal{T}}^j,\qquad
\theta_{\mathrm{large}}^1 = \theta^0 - \eta' g_{\mathcal{T}}^0,
\label{eq:multi-update}
\end{equation}
where $\eta$ is the usual step size and $\eta'$ is a larger one. A single large step rescales one
gradient direction. Repeated small steps recompute $g_{\mathcal{T}}^j$ at each new point, and a stateful
optimizer also updates its state, so the final displacement reflects both accumulated magnitude and a
changing direction. Neither shortcut gives this move for free. A larger learning rate scales every
coordinate alike, even those whose motion is reversing, and more epochs per batch pay one backward pass
per update.
PrePO uses these observed transitions and does not require a closed-form model of the optimizer. In the
following sections, we omit the subscript $\mathcal{T}$ when the training setting is clear.

\section{Motivation: Predict the Endpoint, not Every Intermediate Step}
\label{sec:motivation}

The intermediate states along this path are expensive because each optimizer step needs a backward
pass, yet they do not collect new data; they mainly reveal how the optimizer moves on the current batch.
PrePO treats this path as a sequence of optimizer motions $z^j$ (an AdamW direction, a Muon matrix
displacement, or an SGD step) and asks how each motion turns into the next. Two probe steps answer this
per coordinate. A trend $r\approx1$ means the coordinate keeps its pace, and $0<r<1$ means it is slowing
down. Carrying these trends forward gives the $K$-step destination without taking the remaining steps,
rather than assuming that every coordinate keeps its latest pace. Defining the residual
$e^j=z^{j+1}-Rz^j$ makes $z^{j+1}=Rz^j+e^j$ an identity rather than a model. The endpoint error is
therefore an explicit sum of residuals, bounded by quantities logged during training
(Section~\ref{sec:theory}), and no quadratic objective or frozen Hessian is assumed. The prediction only
selects where the true objective is evaluated again; the committed update always uses a fresh gradient.

\section{PrePO: Predict, Reposition, then Correct}
\label{sec:method}

PrePO starts from a simple observation: when several optimizer updates reuse the same batch, the first
two updates already reveal how the parameters are moving and how that movement is changing. From these
two updates alone, PrePO predicts where $K$ updates would have ended, moves partway toward that
predicted point, and evaluates the objective once there. Only the resulting corrective update is kept,
so an active PrePO step costs three backward passes regardless of~$K$.
The procedure has the same four stages for every optimizer---only the definition of \emph{motion}
(the displacement produced by one update) and \emph{motion trend} (the rate at which successive motions
change) differ. We describe each stage in general terms below, then show the concrete forms for AdamW
and Muon; formal derivations and error bounds appear in
Appendices~\ref{app:optimizer_motion}--\ref{app:stochastic}.

\subsection{Temporary Updates: Learn the Local Direction of Travel}
\label{sec:probe}

Why two updates and not one? One optimizer step shows where the parameters moved, but not whether
that movement is speeding up or slowing down. A second step on the same batch shows exactly that, so
two steps are the minimum that reveals the local \emph{motion trend}. PrePO reuses the fixed batch for
two temporary updates:
\begin{equation}
 (\theta^t,s^t)\xrightarrow{\mathcal O(B)}(\theta^{t,1},s^{t,1})
 \xrightarrow{\mathcal O(B)}(\theta^{t,2},s^{t,2}).
 \label{eq:temporary_updates}
\end{equation}
Each update produces a \emph{motion}, and the change between the two motions is the motion trend. What counts as a ``motion'' depends on the optimizer: for AdamW it is the
adaptive per-coordinate step, for Muon it is the realized matrix displacement, and for SGD it is the
negative gradient step. In every case, both updates are temporary: PrePO saves the original parameters
and optimizer state beforehand and restores them afterward, so the probes measure how the optimizer
is moving without contributing any training progress.

\subsection{Predict: Turn Two Observed Motions into a Virtual $K$-Step Endpoint}
\label{sec:predict}

The two temporary updates give two motions on the same batch, $z_0$ and $z_1$. Inside one PrePO step we
drop the iteration index and write $\theta_0,\theta_1,\theta_2$ for $\theta^t,\theta^{t,1},\theta^{t,2}$.
If the motion keeps changing the same way, where would $K$ updates on this batch end? To answer it, PrePO describes the change from $z_0$ to
$z_1$ with a motion trend $R$, chosen so that $Rz_0\approx z_1$. It helps to think of $R$ as a speed
ratio. A ratio near one means the motion keeps its pace, a ratio between zero and one means it is slowing
down, and a negative ratio means it is going back and forth. Each optimizer measures this ratio at its own
level. AdamW sets the step of every parameter separately, so every parameter gets its own ratio
$r_i=z_{1,i}/z_{0,i}$. Muon updates a whole weight matrix at once, so every matrix gets a single ratio
$\rho$. Once the trend is known, PrePO can estimate the motions it never computes: the third is about
$Rz_1$, the fourth about $R^2z_1$, and so on. These estimates are free, because they reuse only the two
motions already measured. A motion is not yet a change in the parameters, so PrePO applies each estimated
motion the same way the optimizer applies a real one, $\theta_{n+1}=C\theta_n+Az_n$. Here $A$ turns a
motion into a parameter change and $C$ carries the current parameters forward. For AdamW the motion is the
adaptive direction, so $A=-\eta I$ applies the learning rate and $C=I-\eta\lambda I$ applies weight decay.
For Muon the motion is already the full matrix change, so $A=C=I$. Starting from $\theta_1$, PrePO applies
this rule $K-1$ times and calls the result the predicted endpoint $\hat\theta$, or $\hat\theta^{t,K}$ when
the iteration matters. For Muon this is simply the first temporary update plus the estimated motions,
$\hat\theta=\theta_1+\sum_{q=0}^{K-2}\rho^q z_1$. Setting $R=I$ repeats the latest motion, which is weight
prediction. The fitted trend costs nothing beyond the two probes. Where motion keeps its pace it gives the
same endpoint, and where motion slows or reverses it moves that parameter less. The prediction also stays close to what was actually observed. At $K=2$ it gives back the
second temporary update exactly, and for larger $K$ it is exact whenever later motions follow the trend.
When they do not, Section~\ref{sec:theory} bounds the error. Algorithm~\ref{alg:prepo} lists these steps,
Theorem~\ref{thm:adamw_endpoint} and Appendix~\ref{app:muon} derive the endpoint, and
Appendix~\ref{app:retention} the safeguards on $R$.

\subsection{Reposition and Correct}

The predicted endpoint is an estimate, so PrePO does not jump to it directly. It moves a fraction
$\alpha$ of the way there from the original parameters, $\tilde\theta=\theta_0+\alpha(\hat\theta-\theta_0)$.
Reaching this point costs nothing extra, because it only combines parameters that PrePO already
holds. At the new point PrePO computes a fresh gradient of the same batch loss and takes one real
optimizer step from there, starting from the saved optimizer state $s_0$. The prediction only decides where the
loss is evaluated, and the update itself always comes from a real gradient. The two temporary updates are then discarded, and only the corrected
parameters and their optimizer state carry into the next iteration. In total, one PrePO step uses three
backward passes, two for the temporary updates and one for the correction, however large $K$ is.

\begin{algorithm}[t]
\caption{One PrePO step on a fixed batch $B$. Left: the general procedure. Right: what its terms
are for AdamW and Muon. Both loops use the same update rule $\theta_{n+1}=C\theta_n+Az_n$: the probe loop runs it with real
gradients, and the predict loop replays it with predicted motions. The probes are discarded; only the last update is kept.}
\label{alg:prepo}
\small\DontPrintSemicolon\SetCommentSty{algcmt}
\SetKwBlock{Probe}{\textcolor{ProbeC}{1. Probe \normalfont\textit{(2 real updates)}}}{}
\SetKwBlock{Predict}{\textcolor{PredC}{2. Predict \normalfont\textit{(no gradients)}}}{}
\SetKwBlock{Correct}{\textcolor{CorrC}{3. Correct \normalfont\textit{(1 real update)}}}{}
\noindent\begin{minipage}[t]{0.47\linewidth}
\textbf{General procedure} (any optimizer)\;
\KwIn{parameters $\theta_0$, state $s_0$, depth $K$, blend $\alpha$;\\
 batch gradient $g(\theta)=\operatorname{clip}(\nabla\Ls_B(\theta))$}
\vspace{4pt}
\Probe{
 Save $(\theta_0,s_0)$\;
 \For{$n=0,1$}{
  Real step with $g(\theta_n)$: $\theta_n\to\theta_{n+1}$\;
  Record its motion $z_n$ (right panel)\;
 }
}
\vspace{4pt}
\Predict{
 Fit the trend $R$ so that $Rz_0\approx z_1$\;
 \tcp{start at the first probe ($k{=}1$ gives $\theta_2$)}
 $\hat\theta\leftarrow\theta_1$,\ \ $z\leftarrow z_1$\;
 \For{$k=1$ \KwTo $K-1$}{
  Replay one step: $\hat\theta\leftarrow C\hat\theta+Az$\;
  Next motion: $z\leftarrow Rz$\;
 }
 \tcp{$\hat\theta$ is the predicted endpoint}
}
\vspace{4pt}
\Correct{
 Move partway: $\tilde\theta\leftarrow\theta_0+\alpha(\hat\theta-\theta_0)$\;
 Real step from $s_0$: $(\theta,s)\leftarrow\mathcal O(\tilde\theta,s_0,g(\tilde\theta))$\;
}
\end{minipage}\hfill
{\setlength{\fboxsep}{5pt}\colorbox{CardTint}{\begin{minipage}[t]{0.44\linewidth}
\textbf{What the symbols mean}\\[2pt]
{\raggedright Every step has the form $\theta_{n+1}=C\theta_n+Az_n$: $z_n$ is the \emph{motion}, $C$ carries the
old parameters forward, and $A$ scales the motion (Section~\ref{sec:predict}).\par}
\vspace{6pt}
\textbf{AdamW}\\[1pt]
\begin{tabular}{@{}l@{\hspace{0.6em}}l@{}}
$z_n$ & Adam direction $d_n=\hat m_n/(\sqrt{\hat v_n}+\epsilon)$ \\
$C,\,A$ & $(1-\eta\lambda)I$,\ \ $-\eta I$ \\
$R$ & $\operatorname{diag}(r_i)$,\ \ $r_i=\operatorname{clip}(z_{1,i}/z_{0,i})$ \\
\end{tabular}\\[8pt]
\textbf{Muon}\\[1pt]
\begin{tabular}{@{}l@{\hspace{0.6em}}l@{}}
$z_n$ & whole matrix change $\Theta_{n+1}-\Theta_n$ \\
$C,\,A$ & $I$,\ \ $I$ \\
$R$ & $\rho I$,\ \ $\rho=\operatorname{clip}\bigl(\langle z_0,z_1\rangle_F/\|z_0\|_F^2\bigr)$ \\
\end{tabular}\\[8pt]
{\footnotesize\raggedright $\hat m_n,\hat v_n$: bias-corrected Adam moments; $\langle\cdot,\cdot\rangle_F$: Frobenius inner product.
Clip ranges: Appendix~\ref{app:retention}; fallback mechanisms: Appendix~\ref{app:fallback-ablation}.\par}
\end{minipage}}}
\end{algorithm}

\section{Prediction Fidelity and Convergence Guarantees}
\label{sec:theory}

PrePO rests on a prediction, so two questions follow: how close is the predicted endpoint to where $K$
real updates would land, and does training still converge when the prediction is imperfect? For the
first, write each motion as the trend applied to the previous one plus a leftover,
$z_{j+1}^t=Rz_j^t+e_j^t$. This is an identity rather than an assumption, so the endpoint error is an
explicit sum of leftovers. When $\|R\|\le1$, $\|C\|\le1$, and every leftover is at most $\varepsilon_e$,
Theorem~\ref{thm:anchored_endpoint} gives
{\setlength{\abovedisplayskip}{4pt}\setlength{\belowdisplayskip}{4pt}
\begin{equation}
 \bigl\|\theta^{t,K}-\widehat\theta^{t,K}\bigr\|
 \;\le\; \frac{(K-1)(K-2)}{2}\,\|A\|\,\varepsilon_e ,
 \label{eq:main_prediction_bound}
\end{equation}
}
where $\|A\|=1$ for Muon and $\|A\|=\eta$ for AdamW, and Lemma~\ref{lem:derived_residual} computes
$\varepsilon_e$ for AdamW from quantities measured during training. The same identity shows that
prediction is not a larger step. Each coordinate gets its own step size, the full
$\gamma_{\rm hi}=\alpha K\eta+\eta_c$ where its motion persists and at most $2\alpha\eta+\eta_c$ where it
reverses, which keeps PrePO stable at curvatures where a uniformly larger step diverges, as
Proposition~\ref{prop:not_larger_step} shows. Under a Polyak--{\L}ojasiewicz condition, one PrePO step on a batch
therefore makes the progress of $1.5$ to $3.5$ plain steps in our setting (Proposition~\ref{prop:pl_progress}). For the second question,
Theorems~\ref{thm:stochastic_prepo} and~\ref{thm:muon_prepo} prove convergence to a neighbourhood of a stationary point, set by gradient bias and clipping, for
AdamW and Muon with fresh on-policy batches, the clipped surrogate, gradient clipping, the gate, and
optimizer state carried across batches. They assume a smooth target
\citep{Zhang2020GlobalPG,Agarwal2021PGTheory}, bounded gradient bias and clipping drift, and, for AdamW, a
preconditioner held fixed within each outer step. By Corollary~\ref{cor:stochastic_rate}, the AdamW rate is
the usual $O(1/\sqrt T)$ plus two costs of prediction, a variance term from fitting the trend on the same
batch and a drift term bounded by the clip fraction, the first vanishing at $\alpha=0$ and the second whenever no token is clipped. For Muon the rate is
$O(T^{-1/4})$ in the nuclear norm. Appendices~\ref{app:optimizer_motion} and~\ref{app:stochastic}
give the full statements and proofs, and
Table~\ref{tab:prediction-error-validation} compares measured endpoint errors with the bound.

\section{Experiments}
\label{sec:experiments}

\begin{figure}[t]
\centering
\includegraphics[width=\textwidth]{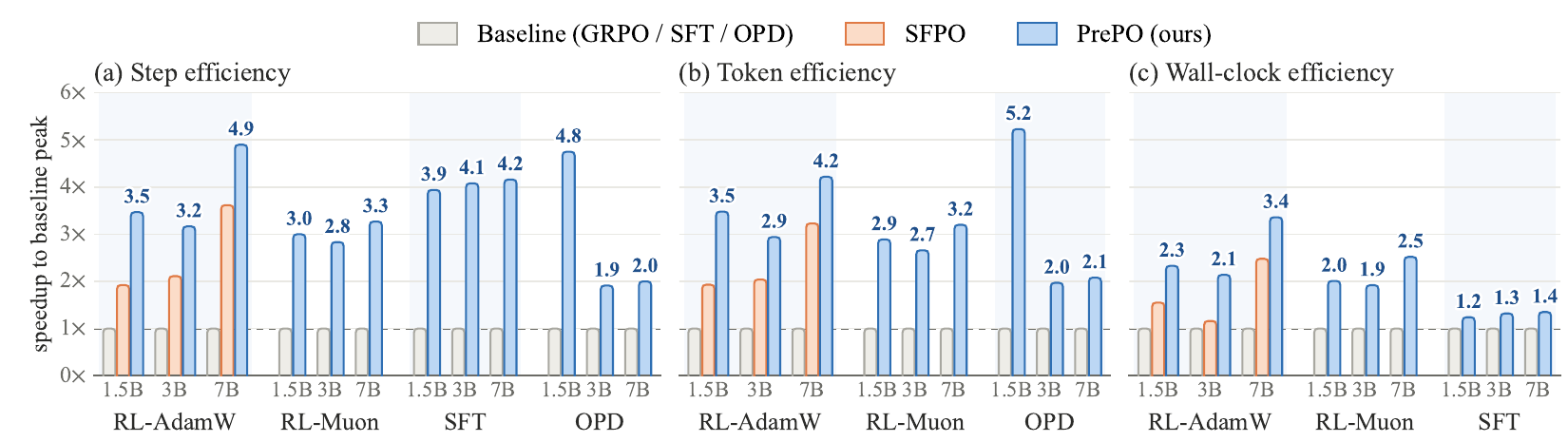}
\caption{Speedup in reaching the baseline's peak for every setting in Tables~\ref{tab:main-results}
and~\ref{tab:generality}; higher is better and the baseline is $1\times$. RL uses Qwen2.5-Math at 1.5B and 7B and Llama-3.2 at 3B, SFT uses Llama-3.2 at 3B and Qwen2.5 otherwise, and OPD uses Qwen2.5. SFT tokens were not recorded, and OPD wall-clock is omitted.}
\label{fig:efficiency}
\end{figure}

\subsection{Experimental Protocol}

We test PrePO where reusing a batch matters most: reinforcement learning with verifiable rewards on
math reasoning. The training pool has 25,417 prompts, 17,917 from DAPO-Math-17K \citep{Yu2025DAPO} and
7,500 from the training split of MATH \citep{Hendrycks2021Math}. The main comparison uses
Qwen2.5-Math-1.5B-Instruct and Qwen2.5-Math-7B-Instruct \citep{QwenMath2024Report} together with
Llama-3.2-3B-Instruct \citep{Meta2024Llama32}, so it covers two model families and three sizes. To keep
the comparison fair, GRPO, SFPO, and PrePO share the same prompt stream, learning rate, rollout budget, reward function, and
hardware, and every main RL run trains for 400 steps with 256 prompts per batch and eight responses per prompt.
PrePO uses $K=5$ with $\alpha=0.5$ for RL, with AdamW or Muon, and $\alpha=0.3$ for SFT and OPD, all fixed before any evaluation. Responses are capped at 3,072 generated tokens in training and evaluation. We evaluate every checkpoint with greedy Pass@1 and
report each method at its final checkpoint on MATH-500 \citep{Hendrycks2021Math,Lightman2024Math500}, AIME24 \citep{AoPS2024AIMEI,AoPS2024AIMEII}, AIME25 \citep{AoPS2025AIMEI,AoPS2025AIMEII}, AMC23
\citep{AoPSAMC}, MinervaMATH \citep{Lewkowycz2022Minerva}, and OlympiadBench
\citep{He2024OlympiadBench}. To see whether the benefit depends on the RL loss or on AdamW, we then change
one ingredient at a time. We replace AdamW with Muon \citep{Jordan2024Muon} on Qwen2.5-Math-1.5B-Instruct, Llama-3.2-3B-Instruct, and Qwen2.5-Math-7B-Instruct
\citep{QwenMath2024Report}, we switch to supervised fine-tuning on the knights-and-knaves logic puzzles
\citep{Xie2025Memorization}, a dataset separate from the math pool, and we switch to on-policy
distillation \citep{Agarwal2024GKD} with a Qwen2.5-Math-7B-Instruct teacher for Qwen2.5-1.5B and Qwen2.5-3B, and a Qwen2.5-32B-Instruct teacher for Qwen2.5-7B, showing how the gain carries across model scale. Appendix~\ref{app:experimental-details} lists every remaining setting.

\begin{table}[H]
\caption{RLVR efficiency and accuracy. Shaded columns give the cost of reaching GRPO's peak score. Scores
use greedy evaluation (Pass@1), with responses capped at 3,072 tokens; Avg.\ is the mean over the six benchmarks.}
\label{tab:main-results}
\begin{center}
\normalsize
\setlength{\tabcolsep}{4pt}
\renewcommand{\arraystretch}{1.35}
\fitwidth{%
\begin{tabular}{l >{\columncolor{EffTint}[\dimexpr\tabcolsep+0.6pt\relax][\dimexpr\tabcolsep+11pt\relax]}c >{\columncolor{EffTint}[\dimexpr\tabcolsep+0.6pt\relax][\dimexpr\tabcolsep+0.6pt\relax]}c >{\columncolor{EffTint}[\dimexpr\tabcolsep+0.6pt\relax][\dimexpr\tabcolsep+0.6pt\relax]}c c ccccc c}
\toprule
 & \multicolumn{3}{c}{\textbf{Efficiency}} & \multicolumn{7}{c}{Accuracy} \\
\cmidrule(lr){2-4}\cmidrule(lr){5-11}
Method & \makecell{Steps to\\GRPO\\Peak $\uparrow$} & \makecell{Token\\Speedup $\uparrow$} & \makecell{Wall-\\clock\\Speedup $\uparrow$} & \makecell{Avg.\\Pass@1\\$\uparrow$} & \makecell{Math\\-500} & \makecell{AIME\\24} & \makecell{AIME\\25} & \makecell{AMC\\23} & \makecell{Minerva\\Math} & \makecell{Olympiad\\Bench} \\
\midrule
\rowcolor{gray!15}
\multicolumn{11}{c}{\textit{Qwen2.5-Math-1.5B}} \\
\midrule
GRPO & 392 & 1.00$\times$ & 1.00$\times$ & 36.29 & 74.00 & 13.33 & 10.00 & 53.01 & 27.94 & 39.47 \\
SFPO & 204\good{1.92\times} & 1.93$\times$ & 1.55$\times$ & 38.50\good{2.21} & 76.40\good{2.40} & 16.67\good{3.34} & 13.33\good{3.33} & 55.42\good{2.41} & 29.41\good{1.47} & 39.76\good{0.29} \\
\textbf{PrePO} & \textbf{113}\good{3.47\times} & \textbf{3.48$\times$} & \textbf{2.33$\times$} & \textbf{40.30}\good{4.01} & \textbf{77.80}\good{3.80} & \textbf{20.00}\good{6.67} & \textbf{16.67}\good{6.67} & \textbf{56.63}\good{3.62} & \textbf{30.51}\good{2.57} & \textbf{40.21}\good{0.74} \\
\midrule
\rowcolor{gray!15}
\multicolumn{11}{c}{\textit{Llama-3.2-3B-Instruct}} \\
\midrule
GRPO & 285 & 1.00$\times$ & 1.00$\times$ & 22.30 & 52.80 & 6.67 & 0.00 & 36.14 & 20.22 & 17.95 \\
SFPO & 135\good{2.11\times} & 2.04$\times$ & 1.16$\times$ & 24.09\good{1.79} & 55.20\good{2.40} & 10.00\good{3.33} & 0.00 & 38.55\good{2.41} & 21.32\good{1.10} & 19.44\good{1.49} \\
\textbf{PrePO} & \textbf{90}\good{3.17\times} & \textbf{2.94$\times$} & \textbf{2.14$\times$} & \textbf{26.25}\good{3.95} & \textbf{57.20}\good{4.40} & \textbf{13.33}\good{6.66} & \textbf{3.33}\good{3.33} & \textbf{40.96}\good{4.82} & \textbf{22.06}\good{1.84} & \textbf{20.62}\good{2.67} \\
\midrule
\rowcolor{gray!15}
\multicolumn{11}{c}{\textit{Qwen2.5-Math-7B-Instruct}} \\
\midrule
GRPO & 387 & 1.00$\times$ & 1.00$\times$ & 45.83 & 85.40 & 20.00 & 16.67 & 72.29 & 35.66 & 44.96 \\
SFPO & 107\good{3.62\times} & 3.23$\times$ & 2.48$\times$ & 46.91\good{1.08} & 85.60\good{0.20} & 23.33\good{3.33} & 16.67 & \textbf{74.70}\good{2.41} & 36.03\good{0.37} & 45.10\good{0.14} \\
\textbf{PrePO} & \textbf{79}\good{4.90\times} & \textbf{4.22$\times$} & \textbf{3.36$\times$} & \textbf{49.17}\good{3.34} & \textbf{87.60}\good{2.20} & \textbf{26.67}\good{6.67} & \textbf{20.00}\good{3.33} & \textbf{74.70}\good{2.41} & \textbf{38.60}\good{2.94} & \textbf{47.48}\good{2.52} \\
\bottomrule
\end{tabular}%
}
\end{center}
\end{table}

\subsection{Fewer Steps, Fewer Rollouts, Less Time}
\label{sec:efficiency}

If one batch moves the model as far as several updates, we need fewer batches, and each batch is one
round of rollouts. Table~\ref{tab:main-results} checks each part against the peak score of GRPO. We count steps as the
number of rollout batches used. PrePO needs $3.2\times$ to $4.9\times$ fewer steps across the three
models. The number of generated tokens falls by almost the same amount (Table~\ref{tab:token-efficiency}). We measure wall-clock time with
rollouts included on four H100 GPUs. It falls by $2.1\times$ to $3.4\times$. The time
gain is smaller than the step gain because a predicting step runs three backward passes. Yet a PrePO step
takes only about $1.5$ times as long as a GRPO step on every model. It is not three times as long because
rollout generation takes most of each step, and PrePO never repeats it. PrePO keeps its saved parameters and optimizer
state in CPU memory, and these times include moving them. SFPO is the baseline that moves farther per batch by running
all $K$ updates for real, and it needs more steps, tokens, and time than PrePO. If $3.5\times$ fewer batches
reach the same score, one update uses only part of what a batch can teach. So the limit is how well we use
each batch and not how much data we collect. Figure~\ref{fig:efficiency} shows that PrePO needs fewer steps and
tokens than every baseline and than SFPO in every setting.

\subsection{Why PrePO Also Finishes Higher}
\label{sec:training_dynamics}

\begin{figure}[t]
\centering
\begin{minipage}[t]{0.49\textwidth}
\centering
\includegraphics[width=\linewidth]{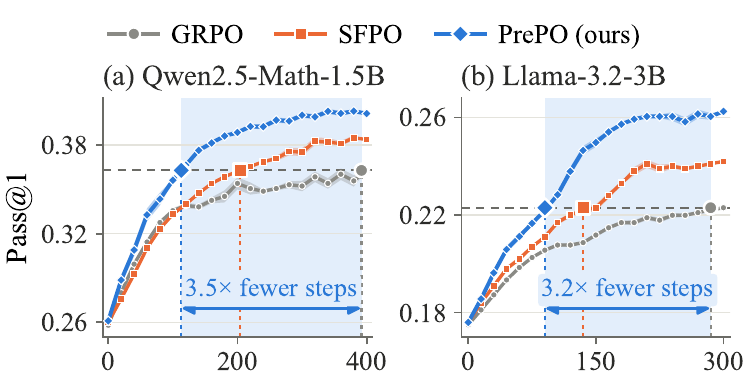}
\caption{Time to reach GRPO's peak, dashed, on Qwen2.5-Math-1.5B, panel a, and Llama-3.2-3B-Instruct,
panel b. Blue shading marks the steps PrePO saves.}
\label{fig:adamw-muon-diagnostics}
\end{minipage}\hfill
\begin{minipage}[t]{0.49\textwidth}
\centering
\includegraphics[width=\linewidth]{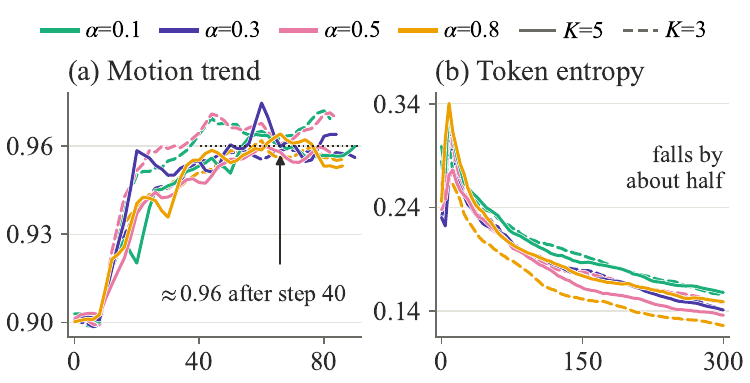}
\caption{Muon runs over training steps: the fitted motion trend, panel a, and token entropy, panel b.
Color marks $\alpha$ and line style marks $K$.}
\label{fig:muon-motion-entropy}
\end{minipage}
\end{figure}

\begin{table}[t]
\caption{Transfer efficiency and accuracy across optimizers and objectives. Shaded columns give the cost
of reaching the baseline's target (in SFT, reaching the baseline's step-400 score under a compute-constrained budget; $n$ is the number of people in a puzzle); OPD wall-clock is approximate. Scores use greedy evaluation (Pass@1; RL responses capped at 3,072 tokens); Avg.\ and Overall are the mean over the six test columns.}
\label{tab:generality}
\label{tab:muon-results}
\label{tab:sft-results}
\begin{center}
\normalsize
\setlength{\tabcolsep}{4pt}
\renewcommand{\arraystretch}{1.12}
\fitwidth{%
\begin{tabular}{l l >{\columncolor{EffTint}[\dimexpr\tabcolsep+0.6pt\relax][\dimexpr\tabcolsep+11pt\relax]}c >{\columncolor{EffTint}[\dimexpr\tabcolsep+0.6pt\relax][\dimexpr\tabcolsep+0.6pt\relax]}c >{\columncolor{EffTint}[\dimexpr\tabcolsep+0.6pt\relax][\dimexpr\tabcolsep+0.6pt\relax]}c c ccccc c}
\toprule
 & & \multicolumn{3}{c}{\textbf{Efficiency}} & \multicolumn{7}{c}{Accuracy} \\
\cmidrule(lr){3-5}\cmidrule(lr){6-12}
\rowcolor{MuonPanel}
\multicolumn{12}{c}{\tableobjective{Reinforcement Learning} \textnormal{(Muon optimizer)}} \\
\midrule
Model & Method & \makecell{Steps to\\GRPO\\Peak $\uparrow$} & \makecell{Token\\Speedup $\uparrow$} & \makecell{Wall-\\clock\\Speedup $\uparrow$} & \makecell{Avg.\\Pass@1\\$\uparrow$} & \makecell{Math\\-500} & \makecell{AIME\\24} & \makecell{AIME\\25} & \makecell{AMC\\23} & \makecell{Minerva\\Math} & \makecell{Olympiad\\Bench} \\
\midrule
\multirow{2}{*}{\makecell{Qwen-\\Math 1.5B}} & GRPO & 345 & 1.00$\times$ & 1.00$\times$ & 37.62 & 75.80 & 10.00 & 13.33 & 55.42 & 29.78 & 41.39 \\
 & \textbf{PrePO} & \textbf{115}\good{3.00\times} & \textbf{2.89$\times$} & \textbf{2.01$\times$} & \textbf{40.34}\good{2.72} & \textbf{78.80}\good{3.00} & \textbf{16.67}\good{6.67} & \textbf{16.67}\good{3.34} & \textbf{56.63}\good{1.21} & \textbf{30.51}\good{0.73} & \textbf{42.73}\good{1.34} \\
\midrule
\multirow{2}{*}{\makecell{Llama\\3B}} & GRPO & 255 & 1.00$\times$ & 1.00$\times$ & 22.85 & 50.20 & \textbf{16.67} & \textbf{10.00} & 26.51 & 18.01 & 15.73 \\
 & \textbf{PrePO} & \textbf{90}\good{2.83\times} & \textbf{2.66$\times$} & \textbf{1.92$\times$} & \textbf{26.44}\good{3.59} & \textbf{51.60}\good{1.40} & 13.33\bad{3.34} & \textbf{10.00} & \textbf{40.96}\good{14.45} & \textbf{20.96}\good{2.95} & \textbf{21.81}\good{6.08} \\
\midrule
\multirow{2}{*}{\makecell{Qwen-\\Math 7B}} & GRPO & 340 & 1.00$\times$ & 1.00$\times$ & 47.57 & 86.60 & 23.33 & 16.67 & 74.70 & 37.50 & 46.59 \\
 & \textbf{PrePO} & \textbf{104}\good{3.27\times} & \textbf{3.20$\times$} & \textbf{2.52$\times$} & \textbf{49.50}\good{1.93} & \textbf{88.00}\good{1.40} & \textbf{26.67}\good{3.34} & \textbf{20.00}\good{3.33} & \textbf{75.90}\good{1.20} & \textbf{38.97}\good{1.47} & \textbf{47.48}\good{0.89} \\
\midrule
\rowcolor{SFTPanel}
\multicolumn{12}{c}{\tableobjective{Supervised Fine-Tuning}} \\
\midrule
Model & Method & \makecell{Step\\Gain $\uparrow$} & \makecell{Token\\Speedup $\uparrow$} & \makecell{Time\\Gain $\uparrow$} & \makecell{Avg.\\Pass@1\\$\uparrow$} & \makecell{IID\\$n{=}3$} & \makecell{IID\\$n{=}4$} & \makecell{IID\\$n{=}5$} & \makecell{IID\\$n{=}6$} & \makecell{OOD\\$n{=}7$} & \makecell{OOD\\$n{=}8$} \\
\midrule
\multirow{2}{*}{\makecell{Qwen\\1.5B}} & SFT & 1.00$\times$ & -- & 1.00$\times$ & 7.33 & 22.00 & 12.00 & 5.00 & 3.00 & 1.00 & 1.00 \\
 & \textbf{PrePO} & \textbf{3.94$\times$} & -- & \textbf{1.24$\times$} & \textbf{18.33}\good{11.00} & \textbf{40.00}\good{18.00} & \textbf{27.00}\good{15.00} & \textbf{19.00}\good{14.00} & \textbf{12.00}\good{9.00} & \textbf{5.00}\good{4.00} & \textbf{7.00}\good{6.00} \\
\midrule
\multirow{2}{*}{\makecell{Llama\\3B}} & SFT & 1.00$\times$ & -- & 1.00$\times$ & 15.33 & 35.00 & 25.00 & 17.00 & 6.00 & 5.00 & 4.00 \\
 & \textbf{PrePO} & \textbf{4.08$\times$} & -- & \textbf{1.32$\times$} & \textbf{33.00}\good{17.67} & \textbf{60.00}\good{25.00} & \textbf{46.00}\good{21.00} & \textbf{36.00}\good{19.00} & \textbf{29.00}\good{23.00} & \textbf{15.00}\good{10.00} & \textbf{12.00}\good{8.00} \\
\midrule
\multirow{2}{*}{\makecell{Qwen\\7B}} & SFT & 1.00$\times$ & -- & 1.00$\times$ & 23.50 & 50.00 & 38.00 & 26.00 & 14.00 & 8.00 & 5.00 \\
 & \textbf{PrePO} & \textbf{4.16$\times$} & -- & \textbf{1.35$\times$} & \textbf{43.33}\good{19.83} & \textbf{74.00}\good{24.00} & \textbf{60.00}\good{22.00} & \textbf{48.00}\good{22.00} & \textbf{38.00}\good{24.00} & \textbf{23.00}\good{15.00} & \textbf{17.00}\good{12.00} \\
\midrule
\rowcolor{OPDPanel}
\multicolumn{12}{c}{\tableobjective{On-Policy Distillation (OPD)}} \\
\midrule
Model & Method & \makecell{Steps to\\OPD\\Peak $\uparrow$} & \makecell{Token\\Speedup $\uparrow$} & \makecell{Wall-\\clock\\Speedup $\uparrow$} & Overall & \makecell{Math\\-500} & \makecell{AIME\\24} & \makecell{AIME\\25} & \makecell{AMC\\23} & \makecell{Minerva\\Math} & \makecell{Olympiad\\Bench} \\
\midrule
\multirow{2}{*}{\makecell{Qwen\\1.5B}} & OPD & 95 & 1.00$\times$ & 1.00$\times$ & 21.56 & 54.40 & 3.33 & 3.33 & 27.71 & 19.49 & 21.07 \\
 & \textbf{PrePO} & \textbf{20}\good{4.75\times} & \textbf{5.23$\times$} & \textbf{1.58$\times$} & \textbf{24.00}\good{2.44} & \textbf{58.40}\good{4.00} & \textbf{6.67}\good{3.34} & \textbf{6.67}\good{3.34} & \textbf{30.12}\good{2.41} & \textbf{19.85}\good{0.36} & \textbf{22.26}\good{1.19} \\
\midrule
\multirow{2}{*}{\makecell{Qwen\\3B}} & OPD & 65 & 1.00$\times$ & \textbf{1.00$\times$} & 26.83 & 62.80 & 6.67 & 3.33 & 33.73 & 24.63 & 29.82 \\
 & \textbf{PrePO} & \textbf{34}\good{1.91\times} & \textbf{1.97$\times$} & 0.79$\times$ & \textbf{29.32}\good{2.49} & \textbf{64.60}\good{1.80} & \textbf{10.00}\good{3.33} & \textbf{6.67}\good{3.34} & \textbf{36.14}\good{2.41} & \textbf{27.94}\good{3.31} & \textbf{30.56}\good{0.74} \\
\midrule
\multirow{2}{*}{\makecell{Qwen\\7B}} & OPD & 52 & 1.00$\times$ & \textbf{1.00$\times$} & 41.27 & 78.60 & 20.00 & 13.33 & 59.04 & 34.19 & 42.43 \\
 & \textbf{PrePO} & \textbf{26}\good{2.00\times} & \textbf{2.08$\times$} & 0.82$\times$ & \textbf{43.10}\good{1.83} & \textbf{80.20}\good{1.60} & \textbf{23.33}\good{3.33} & \textbf{16.67}\good{3.34} & \textbf{60.24}\good{1.20} & \textbf{35.29}\good{1.10} & \textbf{42.88}\good{0.45} \\
\bottomrule
\end{tabular}%
}
\end{center}
\end{table}

PrePO was built to save cost. Yet it also ends with the best average score on every model. All methods share the same fixed budget of 400 steps. In
Figure~\ref{fig:adamw-muon-diagnostics}, GRPO flattens over the second half and peaks late, at step 392 on Qwen2.5-Math-1.5B and 285 on Llama-3.2-3B-Instruct. PrePO reaches that level in a third of the batches or
fewer and keeps improving with the steps it saves. Each PrePO update makes the progress of up to $3.5$
plain updates (Section~\ref{sec:theory}). Near a flat peak a small gain in progress saves many steps, so
the step saving can exceed this, as the $4.9\times$ on Qwen2.5-Math-7B shows. Each coordinate gets its own step size, the full $K$ steps where its motion
persists and a shorter one where the probes show overshooting (Proposition~\ref{prop:not_larger_step}). Figure~\ref{fig:muon-motion-entropy} shows that most motion keeps its direction. The fitted Muon trend stays near $0.96$ and almost no
matrix has a negative trend. The best $\alpha$ depends on $K$. It is $0.5$ at $K=3$, and Muon prefers
$0.3$ at $K=5$, but the main runs keep the pre-fixed $\alpha=0.5$ (Appendix~\ref{app:optimizer-k-alpha-ablation}). Token entropy falls as accuracy
rises \citep{Cui2025EntropyMechanism,Wang2026SFPO}, and a sharp change in it sets $\alpha$ to zero for the rest of training. Removing this fallback drops the mean accuracy from $38.9$ to $31.1$, and a temporary one reaches $37.7$ (Appendix~\ref{app:fallback-ablation}).

\subsection{Transfer Beyond RLVR and Across Optimizers}
\label{sec:transfer}

The probes only read how the optimizer moves, so the gain should not depend on the RL loss or on AdamW.
We swap AdamW for Muon and also run SFT and on-policy distillation, each against its own baseline. A code-generation run (Appendix~\ref{app:code}) extends RLVR beyond math.
In Figure~\ref{fig:diagnostics-transfer} and Table~\ref{tab:generality}, PrePO reaches every
target $1.9\times$ to $4.8\times$ sooner in steps. It also ends higher in every case, most clearly in SFT. With the same data, learning rate of $10^{-5}$, and 400 steps, the SFT baseline reaches only $7.3$, $15.3$, and $23.5$ average accuracy. PrePO
reaches $18.3$, $33.0$, and $43.3$. PrePO reaches the baseline's step-400 score in only 96 to 102 steps. The gain also holds on test puzzles larger than any training puzzle. In SFT each step is one pass over a batch, so the $4\times$ step saving reaches the same score from a quarter of the examples. Without a costly rollout round a PrePO step costs about three plain steps, so the $4\times$ step saving
becomes $1.2\times$ to $1.4\times$ in time (Table~\ref{tab:generality}). OPD samples one response per prompt, so a PrePO step there costs 2.4 to 3 plain steps, more than its $2\times$ saving on 3B and 7B.

\begin{figure}[t]
\centering
\includegraphics[width=\textwidth]{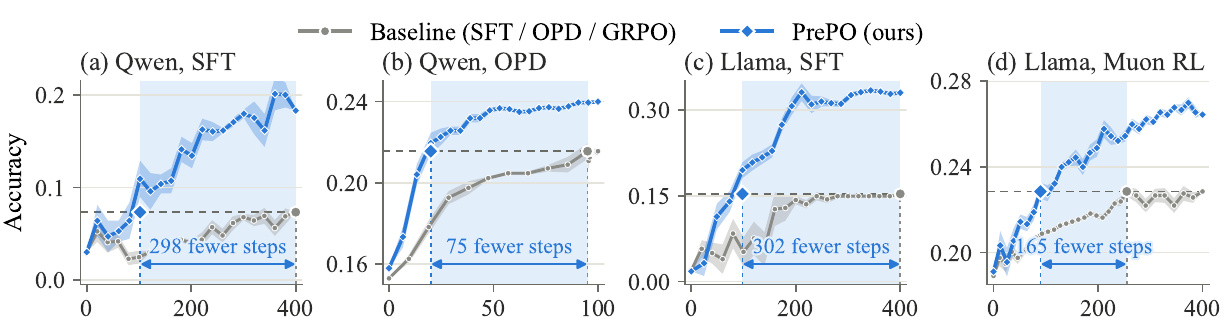}
\caption{Time to each baseline's target, dashed (its peak; for SFT its step-400 score): SFT and OPD on Qwen2.5-1.5B, a and b,
and SFT and Muon RL on Llama-3.2-3B, c and d. Shading marks the steps PrePO saves.}
\label{fig:diagnostics-transfer}
\end{figure}

\subsection{Prediction Holds Across Model Size}
\label{sec:horizon-main}

\begin{wrapfigure}{r}{0.5\linewidth}
\vspace{-12pt}
\centering
\includegraphics[width=\linewidth]{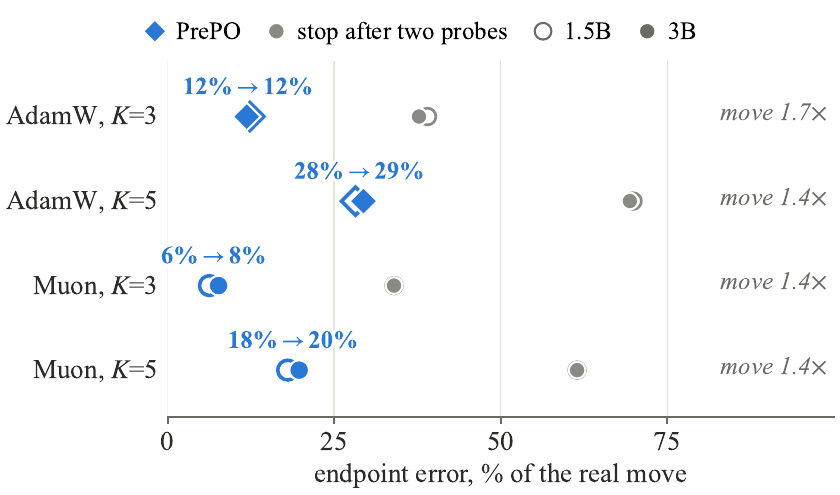}
\vspace{-16pt}
\caption{Median endpoint error as a share of the real $K$-step move, 1.5B hollow and 3B filled.}
\label{fig:prediction-stability}
\vspace{-10pt}
\end{wrapfigure}
Looking $K$ steps ahead does not cost $K$ steps of compute. SFPO pays $K{+}1$ backward
passes per update. PrePO pays three for any $K$. Our logs confirm this at $K=5$ and $K=10$. Nor is the
path too irregular to predict. We run the real $K$ steps next to the prediction on the same batch. The prediction points the same way
as the real move, with a cosine above 0.9 even at $K=10$. The prediction also lands two to five
times closer than stopping after the two probes. We then move from Qwen2.5-Math-1.5B to Qwen2.5-3B. The real move becomes 1.4 to 1.7 times longer. Yet the error of PrePO stays the same share of
that move (Figure~\ref{fig:prediction-stability}). So the error depends on the horizon and not on model
size, as the bound predicts (Appendix~\ref{app:prediction-error-validation}).

\section{Discussion and Conclusion}
\label{sec:discussion}

Over a few updates the optimizer path is regular enough to predict
from two probes. What sets it apart is the cost. SFPO runs every inner step, so its cost grows with $K$. PrePO
looks $K$ steps ahead for three backward passes and lets a real gradient make the update.  Our theory shows that prediction keeps the usual convergence rate and that one PrePO step makes up to
$3.5$ plain steps of progress, and our experiments confirm the speedup. PrePO needs fewer steps and ends higher in every setting. It helps most where batches are costly or budgets
are tight.

\clearpage
\subsection*{AI use statement}
We used generative AI tools to obtain feedback on experimental design and to assist with the
analysis of experimental findings. These tools also supported literature searches, analysis code,
figure preparation, and the drafting and editing of parts of the manuscript. The authors reviewed all
AI-assisted material and take responsibility for the methods, results, claims, code, and figures in this
paper.

\subsection*{Ethics statement}
PrePO lowers the cost of post-training without changing the data, reward, loss,
or optimizer. In our main RL runs it needs a third of the rollout batches or fewer, which saves GPU hours
and energy and brings RL post-training within reach of smaller budgets. Cheaper training also makes
misuse cheaper, so the usual safeguards for releasing models still apply.

\subsection*{Reproducibility statement}
All tables and figures in this paper are reproducible from the settings in
Appendix~\ref{app:experimental-details}, and we will release the code, scripts, and random seeds used to produce them.
The assumptions and complete proofs of all theoretical results are in Appendices~\ref{app:optimizer_motion}
and~\ref{app:stochastic}.

\bibliography{iclr2026_conference}
\bibliographystyle{iclr2027_conference}

\appendix
\makeatletter
\g@addto@macro\normalsize{\setlength\abovedisplayskip{9pt plus 3pt minus 2pt}\setlength\belowdisplayskip{9pt plus 3pt minus 2pt}\setlength\abovedisplayshortskip{5pt plus 2pt}\setlength\belowdisplayshortskip{7pt plus 2pt}}
\makeatother
\normalsize
\setlength{\parskip}{8pt plus 1pt}
\setlength{\topsep}{9pt plus 2pt minus 1pt}
\linespread{1.06}\selectfont
\section{Endpoint Prediction: Identity, Bound, and Optimizer Forms}
\label{app:optimizer_motion}

This appendix writes PrePO's endpoint error exactly, bounds it, and gives the AdamW and Muon forms of the predictor. Appendix~\ref{app:stochastic} then bounds the residuals and proves convergence under RLVR training.

\subsection{Reference trajectory and notation}

Fix $K\ge2$, a starting point $(\theta_0,s_0)$, a batch, and all randomness. The reference trajectory applies $K$ real transitions $(\theta_{n+1},s_{n+1})=\Phi(\theta_n,s_n,\bar g_n)$ on this batch; the first two are the probes, so $\theta_1,\theta_2$ and the motions $z_0,z_1$ are observed. PrePO discards the probe states and restores $s_0$ before the correction, so the probes reach the committed update only through the repositioned point $\widetilde\theta$. Norms are Euclidean for vectors, spectral for operators, and Frobenius for matrix motions.

\begin{assumption}[Fixed transition representation]
\label{ass:motion_representation}
Fixed matrices $C,A\in\mathbb R^{d\times d}$ and motions $z_n\in\mathbb R^d$ represent the transitions as $\theta_{n+1}=C\theta_n+Az_n$ for $0\le n<K$.
\end{assumption}
Fit a trend $R$ from the two observed motions and hold it fixed. The residual and the propagation operator are
\begin{equation}
 e_n:=z_{n+1}-Rz_n,\qquad
 M_m:=\sum_{q=0}^{m-1}C^{m-1-q}AR^q,\qquad M_0=0.
 \label{eq:propagation_operator}
\end{equation}
So $z_{n+1}=Rz_n+e_n$ holds by definition, not by assumption. For displacement motions ($C=A=I$), $M_m=S_m(R)=\sum_{q<m}R^q$, which for a scalar $r\ne1$ is $(1-r^m)/(1-r)$ and equals $m$ at $r=1$.

\subsection{Endpoint identity and bound}

\begin{theorem}[Anchored endpoint decomposition]
\label{thm:anchored_endpoint}
Under Assumption~\ref{ass:motion_representation}, for $0\le j<K$ let $\widehat\theta_K^{(j)}=C^{K-j}\theta_j+M_{K-j}z_j$. For any fixed $R$,
\begin{equation}
 \theta_K-\widehat\theta_K^{(j)}
 =\sum_{\ell=j}^{K-2}M_{K-1-\ell}e_\ell,\qquad
 \|\theta_K-\widehat\theta_K^{(j)}\|
 \le\sum_{\ell=j}^{K-2}\|M_{K-1-\ell}\|\|e_\ell\|.
 \label{eq:j_anchor_error}
\end{equation}
\end{theorem}
\begin{proof}
Unrolling both recurrences from step $j$ gives
\begin{equation}
 z_n=R^{n-j}z_j+\sum_{\ell=j}^{n-1}R^{n-1-\ell}e_\ell,\qquad
 \theta_K=C^{K-j}\theta_j+\sum_{n=j}^{K-1}C^{K-1-n}Az_n.
 \label{eq:motion_unroll}
\end{equation}
Substitute the first into the second. The $z_j$ terms sum to $M_{K-j}z_j$, and after swapping the double sum each $e_\ell$ collects the coefficient $\sum_{q=0}^{K-2-\ell}C^{K-2-\ell-q}AR^q=M_{K-1-\ell}$. The bound is the triangle inequality.
\end{proof}

PrePO anchors at the second motion, $j=1$:
\begin{equation}
 \widehat\theta_K=C^{K-1}\theta_1+M_{K-1}z_1,\qquad
 \theta_K-\widehat\theta_K=\sum_{\ell=1}^{K-2}M_{K-1-\ell}e_\ell.
 \label{eq:prepo_final_predictor}
\end{equation}
This anchor needs no inverse and carries no $e_0$ term. It is exact at $K=2$, and at $K=3$ the whole error is $Ae_1$.

For scalars write $T_m(a,b)=\sum_{q<m}a^{m-1-q}b^q$, which equals $(a^m-b^m)/(a-b)$, or $ma^{m-1}$ when $a=b$. With $c_*=\|C\|$ and $r_*=\|R\|$, submultiplicativity gives $\|M_m\|\le\|A\|T_m(c_*,r_*)$, so
\begin{equation}
 \|\theta_K-\widehat\theta_K\|
 \le\|A\|\sum_{\ell=1}^{K-2}T_{K-1-\ell}(c_*,r_*)\|e_\ell\|.
 \label{eq:general_explicit_bound}
\end{equation}
When $c_*,r_*\le1$ each term of $T_m$ is at most one, and if every $\|e_\ell\|\le\varepsilon_e$,
\begin{equation}
 \|\theta_K-\widehat\theta_K\|
 \le\tfrac12(K-1)(K-2)\|A\|\varepsilon_e.
 \label{eq:stable_general_bound}
\end{equation}
Clipping a diagonal or scalar trend to $[-1,1]$ gives $\|R\|\le1$. The residual bound $\varepsilon_e$ is not left as an assumption; Lemma~\ref{lem:derived_residual} computes it from optimizer constants and quantities logged during training. Some such information is needed, because the probes alone cannot bound the error:

\begin{proposition}[Two-probe identifiability limit]
\label{prop:two_probe_limit}
For unrestricted continuations and $K\ge3$, no finite error bound depends only on $\theta_0,z_0,z_1$, even if $e_0=0$.
\end{proposition}
\begin{proof}
Take scalar $C=A=R=1$, $z_0=z_1=1$, and $z_n=1+v$ for $n\ge2$. The probes do not see $v$, yet the error is $(K-2)v$.
\end{proof}

\subsection{Retention fit}
\label{app:retention}

One pair $(z_0,z_1)$ cannot identify a full matrix $R$, so PrePO fits a diagonal trend for AdamW and a scalar trend for Muon. The diagonal fit takes the coordinatewise ratio, projected onto $[r_{\min},r_{\max}]$, and falls back to $r_{\rm neu}$ where the first motion is too small to divide by:
\begin{equation}
 r_i=\begin{cases}
 \Pi_{[r_{\min},r_{\max}]}(z_{1,i}/z_{0,i}),&|z_{0,i}|>\delta,\\
 r_{\rm neu},&\text{otherwise},
 \end{cases}
 \qquad R=\operatorname{diag}(r_i).
 \label{eq:structured_retention}
\end{equation}
The projected ratio is the constrained least-squares fit of $z_{1,i}\approx rz_{0,i}$. The scalar fit is $\rho=\langle z_0,z_1\rangle/\|z_0\|^2$, projected the same way and set to $0$ when $z_0=0$. We use $[r_{\min},r_{\max}]=[-1,1]$, $r_{\rm neu}=1$, and $\delta=10^{-3}\operatorname{RMS}(z_0)$. The interval contains the multiplier $1-\eta h$ of every stable gradient step, so on a quadratic the clip never binds while the probes are stable (Proposition~\ref{prop:not_larger_step}). When it does bind, the prediction becomes conservative: growing motion is extrapolated at constant speed.

\subsection{AdamW form}
\label{app:adamw}

\begin{assumption}[Decoupled AdamW transition]
\label{ass:adamw_writeback}
Every modeled parameter in group $p$ is updated at every reference step with fixed learning rate $\eta_p>0$ and weight decay $\lambda_p\ge0$, through its adaptive direction $d_n$:
\begin{equation}
 \theta_{n+1}=c_p\theta_n-\eta_pd_n,\qquad c_p=1-\eta_p\lambda_p.
 \label{eq:adamw_writeback}
\end{equation}
\end{assumption}
This is decoupled AdamW \citep{Loshchilov2019AdamW}. The direction is read back from the parameters,
\begin{equation}
 d_n=\bigl[(\theta_n-\theta_{n+1})-\eta_p\lambda_p\theta_n\bigr]/\eta_p,
 \label{eq:adamw_reconstruction}
\end{equation}
so it already contains the moments, bias correction, and epsilon. The ratio $d_{1,i}/d_{0,i}$ describes adaptive motion, not a Hessian entry.

\begin{theorem}[AdamW endpoint decomposition]
\label{thm:adamw_endpoint}
Under Assumption~\ref{ass:adamw_writeback}, let $R=\operatorname{diag}(r_i)$, $e_n=d_{n+1}-Rd_n$, and $T_m(c,R)=\sum_{q=0}^{m-1}c^{m-1-q}R^q$. Within group $p$,
\begin{align}
 \widehat\theta_K^{\rm AdamW}&=c_p^{K-1}\theta_1-\eta_pT_{K-1}(c_p,R)d_1,\nonumber\\
 \theta_K-\widehat\theta_K^{\rm AdamW}&=-\eta_p\sum_{\ell=1}^{K-2}T_{K-1-\ell}(c_p,R)e_\ell,
 \label{eq:adamw_endpoint}
\end{align}
with norm at most $\eta_p\sum_{\ell=1}^{K-2}T_{K-1-\ell}(|c_p|,\|R\|)\|e_\ell\|$.
\end{theorem}
\begin{proof}
Put $C=c_pI$, $A=-\eta_pI$, and $z_n=d_n$ in Eq.~(\ref{eq:prepo_final_predictor}); then $M_m=-\eta_pT_m(c_p,R)$.
\end{proof}
Several groups give block-diagonal $C$ and $A$, which a diagonal $R$ preserves. The quadratic-in-$K$ bound needs $|c_p|\le1$ and $\max_i|r_i|\le1$.

\subsection{Muon form}
\label{app:muon}

For a matrix $\Theta_n$, take the total motion $U_n=\Theta_{n+1}-\Theta_n$, so $C=A=I$. With a scalar trend $\rho$, $E_n=U_{n+1}-\rho U_n$, and $\Delta_K=\Theta_K-\Theta_0$, Eq.~(\ref{eq:prepo_final_predictor}) becomes
\begin{equation}
 \widehat\Delta_K=U_0+S_{K-1}(\rho)U_1,\qquad
 \Delta_K-\widehat\Delta_K=\sum_{\ell=1}^{K-2}S_{K-1-\ell}(\rho)E_\ell,
 \label{eq:muon_endpoint}
\end{equation}
and the error is at most $\sum_{\ell}S_{K-1-\ell}(|\rho|)\|E_\ell\|_F$. Muon's orthogonalized motions have nearly equal norm, and then the fitted trend is the cosine between $U_0$ and $U_1$. The fitted value near $0.96$ in Section~\ref{sec:training_dynamics} therefore means consecutive Muon motions are nearly aligned.

\section{Stochastic, On-Policy, and Clipped Training}
\label{app:stochastic}

RLVR training breaks the clean setting above in five ways: every outer step draws a fresh batch from the current or a recent policy, the clipped surrogate is not smooth, gradients are norm-clipped, the preconditioner is carried across batches, and the gate can switch prediction off. We only assume the training target $f$ is smooth. Lemma~\ref{lem:derived_residual} bounds the residuals that Appendix~\ref{app:optimizer_motion} needed, and Theorems~\ref{thm:stochastic_prepo} and~\ref{thm:muon_prepo} prove convergence.

\paragraph{Protocol.} Outer step $t$ starts at $(x_t,s_t)$; $\mathcal F_t$ holds everything seen before batch $B_t$ is drawn. On $B_t$, $\ell_t$ is the objective of Section~\ref{sec:background} with rollouts and advantages fixed, and the supplied gradient is the clipped $\bar g_t(\theta)=\Pi_c(\nabla\ell_t(\theta))$, $\Pi_c(v)=\min\{1,c/\|v\|\}v$. The optimizer moves by $\theta\leftarrow\theta-\eta P_tm$, where $P_t$ is a positive diagonal preconditioner read from $s_t$ and held fixed during the step, and $m\leftarrow\beta m+(1-\beta)\bar g_t(\theta)$ starts from the stored $m_t$. The probes use this rule; then $m_t$ is restored and the correction is $m_{t+1}=\beta m_t+(1-\beta)\bar g_t(\widetilde x_t)$, $x_{t+1}=\widetilde x_t-\eta_cP_tm_{t+1}$. For AdamW, $P_t=\operatorname{diag}(1/(\sqrt{\widehat v_t}+\epsilon_{\rm opt}))$. We ignore AdamW's within-step update of $\widehat v$, which changes each denominator by less than $0.1\%$ for $\beta_2=0.999$ when $\bar g_i^2\le3v_i$. Clipping gives $1/(c+\epsilon_{\rm opt})\le p_{\rm lo}$ and $p_{\rm hi}\le1/\epsilon_{\rm opt}$. Weight decay is added after Corollary~\ref{cor:stochastic_rate}.

\begin{assumption}[Stochastic protocol]
\label{ass:stochastic}
\begin{itemize}\setlength{\itemsep}{2pt}
\item[\textbf{(a)}]~$f:\mathbb R^d\to\mathbb R$ has $L$-Lipschitz gradient and $f_{\inf}=\inf f>-\infty$.
\item[\textbf{(b)}]~$G_t:=\bar g_t(x_t)$ satisfies $\mathbb E[G_t\mid\mathcal F_t]=\nabla f(x_t)+b_t$ and $\mathbb E[\|G_t-\mathbb E[G_t\mid\mathcal F_t]\|^2\mid\mathcal F_t]\le\sigma^2$, with $b_t$ $\mathcal F_t$-measurable and $\|b_t\|\le\bar b$.
\item[\textbf{(c)}]~For a constant $L_B$ and a random variable $0\le J_t\le\bar J$, $\|\bar g_t(\theta)-\bar g_t(\theta')\|\le L_B\|\theta-\theta'\|+J_t$ for every pair of points $\theta,\theta'$ used in step $t$.
\item[\textbf{(d)}]~$p_{\rm lo}I\preceq P_t\preceq p_{\rm hi}I$ with $p_{\rm lo}>0$ and $\kappa_P=p_{\rm hi}/p_{\rm lo}$; the retention $R_t=\operatorname{diag}(r_{t,i})$, which may be fitted on $B_t$ by Eq.~(\ref{eq:structured_retention}), has $r_{t,i}\in[r_{\rm lo},1]$ with $-1\le r_{\rm lo}\le1$ (we use $r_{\rm lo}=-1$).
\item[\textbf{(e)}]~The gate is constant, $\alpha_t\equiv\alpha\in[0,1]$; one switch to $\alpha=0$ is treated after Theorem~\ref{thm:stochastic_prepo}.
\item[\textbf{(f)}]~Gradients are clipped at a norm $c<\infty$, the first moment starts at $m_0=0$, and $\|\nabla f\|\le G_f$.
\end{itemize}
\end{assumption}
In RLVR, $f$ is the negative expected verifiable reward, which is $L$-smooth under bounded rewards and bounded derivatives of $\log\pi_\theta$ \citep{Zhang2020GlobalPG,Agarwal2021PGTheory}. The bias $b_t$ absorbs group normalization, token averaging, clipping, and staleness, and the one-update baseline carries the same bias. The per-batch loss need not be smooth; the next lemma gives (c) for the clipped surrogate.

\begin{lemma}[Clipped-surrogate drift]
\label{lem:clip_drift}
Let $\ell_t(\theta)=-\sum_iw_i\min\{\rho_iA_i,\operatorname{clip}(\rho_i,1-\epsilon,1+\epsilon)A_i\}+\ell_s(\theta)$ with token weights $w_i\ge0$ summing to one, advantages $A_i$, ratios $\rho_i(\theta)=\pi_\theta(o_i)/\pi_{\rm old}(o_i)$, and smooth extra terms $\ell_s$. Let $a_i(\theta)=0$ when token $i$ is clipped ($A_i>0,\rho_i>1+\epsilon$ or $A_i<0,\rho_i<1-\epsilon$) and $a_i(\theta)=1$ otherwise. The supplied gradient is
\begin{equation}
 \nabla\ell_t(\theta)=-\sum_iw_ia_i(\theta)A_i\nabla\rho_i(\theta)+\nabla\ell_s(\theta).
 \label{eq:clipped_gradient}
\end{equation}
If, on a convex set containing the points of step $t$, $\nabla\rho_i$ is $L_i$-Lipschitz with $\|\nabla\rho_i\|\le\Gamma_i$ and $\nabla\ell_s$ is $L_s$-Lipschitz, then Assumption~\ref{ass:stochastic}(c) holds with
\begin{equation}
 L_B=\sum_iw_i|A_i|L_i+L_s,\qquad
 J_t=\sup_{\theta,\theta'}\sum_{i:\,a_i(\theta)\ne a_i(\theta')}w_i|A_i|\Gamma_i
 \le2\kappa_t\max_i|A_i|\Gamma_i,
 \label{eq:clip_drift}
\end{equation}
where $\kappa_t$ is the largest clip fraction $\kappa(\theta)=\sum_iw_i(1-a_i(\theta))$ among those points. If the batch is on-policy and every point stays within $\log(1+\epsilon)/\Gamma_s$ of $x_t$, with $\Gamma_s$ bounding $\|\nabla\log\pi_\theta(o_i)\|$, no token is clipped and $J_t=0$.
\end{lemma}
\begin{proof}
Each token's term is either $A_i\rho_i$ or a constant, which gives Eq.~(\ref{eq:clipped_gradient}). Between two points, tokens unclipped at both contribute at most $w_i|A_i|L_i\|\theta-\theta'\|$, tokens clipped at both contribute nothing, and tokens clipped at exactly one contribute at most $w_i|A_i|\Gamma_i$. Norm clipping is nonexpansive, so the bound carries over to $\bar g_t$. For the last claim, the mean-value theorem gives $|\log\rho_i|\le\log(1+\epsilon)$, which keeps $\rho_i$ inside the clip range.
\end{proof}
With token-mean aggregation, $\kappa(\theta)$ is exactly the clip fraction every RLVR trainer already logs (\texttt{pg\_clipfrac} in \textsc{verl}), at both probes and the repositioned point. So $J_t$ is a measured quantity, and it is zero in any step whose three passes clip nothing.

\begin{lemma}[Derived motion residuals]
\label{lem:derived_residual}
Run the reference trajectory of step $t$ from $\theta_0=x_t$ with the protocol above, with $\|m_t\|\le c$. Let $a=\eta p_{\rm hi}L_B$ and $\Lambda_t=\|m_t-G_t\|$, the gap between the stored first moment and the fresh gradient. For $0\le\ell\le K-2$ the residuals $e_\ell=z_{\ell+1}-R_tz_\ell$ satisfy
\begin{equation}
 \|e_\ell\|\le p_{\rm hi}\bigl[(1-\beta)\beta^{\ell+1}\Lambda_t+(1-\beta)J_t
 +(1-\beta^{\ell+1})ac+(1-r_{\rm lo})c\bigr],
 \label{eq:derived_residual}
\end{equation}
and the second-motion predictor $\widehat\theta_K=\theta_1-\eta S_{K-1}(R_t)z_1$ satisfies
\begin{equation}
 \|\theta_K-\widehat\theta_K\|\le\eta p_{\rm hi}\Bigl[(K-2)\beta^2\Lambda_t
 +\tfrac12(K-1)(K-2)\bigl((1-\beta)J_t+\eta p_{\rm hi}L_Bc+(1-r_{\rm lo})c\bigr)\Bigr].
 \label{eq:derived_endpoint}
\end{equation}
\end{lemma}
\begin{proof}
Split $e_\ell=(1-\beta)P_t(\bar g_{\ell+1}-m^{(\ell)})+(I-R_t)P_tm^{(\ell)}$. Every first moment is an average of clipped gradients, so the second part is at most $(1-r_{\rm lo})p_{\rm hi}c$ and each step moves at most $\eta p_{\rm hi}c$. Writing $m^{(\ell)}$ as a weighted average of $m_t$ and $\bar g_0,\dots,\bar g_\ell$ and applying (c) to each gap gives $\|\bar g_{\ell+1}-m^{(\ell)}\|\le\beta^{\ell+1}\Lambda_t+J_t+ac(1-\beta^{\ell+1})/(1-\beta)$, which is Eq.~(\ref{eq:derived_residual}). Summing with the weights $K-1-\ell$ of Eq.~(\ref{eq:general_explicit_bound}) gives the endpoint bound.
\end{proof}
Every term is an optimizer constant, the retention interval, or something logged during training: the momentum lag $\Lambda_t$, the clip fraction inside $J_t$, and $L_B$. The probes alone cannot bound the error (Proposition~\ref{prop:two_probe_limit}); the probes plus smoothness and clipping information can. For SGD on a quadratic, the ratio fit recovers $R=I-\eta H$ and every residual vanishes. Muon's orthogonalization has no uniform Lipschitz constant, so for Muon we get only $\|E_\ell\|_F\le(1+|\rho|)\max_n\|U_n\|_F$ and report the fitted retention instead (Section~\ref{sec:training_dynamics}).

\begin{lemma}[A PrePO step is a momentum step with a per-coordinate step size]
\label{lem:inexact_step}
Under Assumption~\ref{ass:stochastic}(c)--(f), let $S_t=S_{K-1}(R_t)=\sum_{q=0}^{K-2}R_t^q$, $\gamma_{\rm lo}=\alpha\eta+\eta_c$, $\gamma_{\rm hi}=\alpha K\eta+\eta_c$, $a=\eta p_{\rm hi}L_B$, and $\varrho=\alpha(K-1)/(\alpha+\tau)$ with $\tau=\eta_c/\eta$. Then $S_t$ is diagonal with entries in $[0,K-1]$, and at most $1$ wherever $r_{t,i}\le0$. Moreover,
\begin{equation}
\begin{aligned}
 x_{t+1}&=x_t-P_t\bigl[(\beta\gamma_{\rm lo}I+E_t)m_t+((1-\beta)\gamma_{\rm lo}I+F_t)G_t+W_t\bigr],\\
 m_{t+1}&=\beta m_t+(1-\beta)(G_t+w_t),\\
 E_t&=\alpha\eta\beta^2S_t,\quad F_t=\alpha\eta(1-\beta^2)S_t,\quad
 W_t=\alpha\eta(1-\beta)S_t\delta_1+\eta_c(1-\beta)\delta_c,\quad w_t=\delta_c,
\end{aligned}
 \label{eq:trend_form}
\end{equation}
where $\delta_1=\bar g_t(\theta_1)-G_t$ and $\delta_c=\bar g_t(\widetilde x_t)-G_t$ satisfy $\|\delta_1\|\le ac+J_t$ and $\|\delta_c\|\le\alpha Kac+J_t$. Also $\|x_{t+1}-x_t\|\le\gamma_{\rm hi}p_{\rm hi}c$, $\|m_{t+1}\|\le c$, and
\begin{equation}
 \|W_t+\beta\gamma_{\rm lo}w_t\|\le\gamma_{\rm lo}U_t,\qquad U_t:=\varrho(1-\beta)(ac+J_t)+\alpha Kac+J_t .
 \label{eq:trend_bounds}
\end{equation}
Coordinate $i$ moves with total step $\Gamma_{t,i}=\alpha\eta(1+S_{t,ii})+\eta_c\in[\gamma_{\rm lo},\gamma_{\rm hi}]$, and its momentum weight lies in $[\beta^2,\beta]$.
\end{lemma}
\begin{proof}
The range of $S_t$ follows from $S_n(r)(1-r)=1-r^n\in[0,2]$ for $r\in[-1,1)$. The two probes give $z_0=P_t\bar m$ and $z_1=P_tm^{(1)}$ with $\bar m=\beta m_t+(1-\beta)G_t$ and $m^{(1)}=\beta\bar m+(1-\beta)(G_t+\delta_1)$, so the predicted move is $-\eta P_t(\bar m+S_tm^{(1)})$. Scale it by $\alpha$, add the correction, and collect the coefficients of $m_t$, $G_t$, $\delta_1$, and $\delta_c$; this is Eq.~(\ref{eq:trend_form}). The norm bounds follow because every first moment has norm at most $c$, and (c) then bounds $\delta_1$ and $\delta_c$.
\end{proof}
So the trend never enters as an error. A coordinate whose motion persists ($r=1$) takes the full step $\gamma_{\rm hi}$, and one whose motion reverses takes at most $2\alpha\eta+\eta_c$. With $R=I$, Eq.~(\ref{eq:trend_form}) is a quasi-hyperbolic momentum step \citep{Ma2019QHM}. This is what separates PrePO from a larger learning rate:

\begin{proposition}[Prediction is not a larger step]
\label{prop:not_larger_step}
\textbf{(i)}~A learning rate scaled so that one update has step $\gamma_{\rm hi}$, or prediction at constant speed ($R=I$, weight prediction), gives every coordinate the step $\gamma_{\rm hi}$. PrePO gives coordinate $i$ the step $\alpha\eta(1+S_{K-1}(r_{t,i}))+\eta_c$, which is $\gamma_{\rm hi}$ only where $r_{t,i}=1$ and at most $2\alpha\eta+\eta_c$, for every $K$, where $r_{t,i}\le0$.
\textbf{(ii)}~On one coordinate of a quadratic with curvature $h\ge0$ and plain gradient motion, the fitted trend is $r=1-\eta h$ and the predictor returns the true $K$-step endpoint $r^K\theta_0$. One PrePO step multiplies the coordinate by
\[
 \varphi(h)=(1-\eta_ch)\bigl(1-\alpha+\alpha(1-\eta h)^K\bigr),\qquad \varphi'(0)=-\gamma_{\rm hi},
\]
and $|\varphi(h)|\le1$ for $0\le h\le2/\max(\eta,\eta_c)$. A single step of size $\gamma_{\rm hi}$ diverges once $h>2/\gamma_{\rm hi}$. Weight prediction repositions by $1-\alpha+\alpha(1-\eta h)(1-(K-1)\eta h)$, which exceeds $1$ once $\eta h>K/(K-1)$, and at $\eta h=\eta_ch=2$ its full step has magnitude $1+2\alpha(K-2)$.
\end{proposition}
\begin{proof}
(i) is Lemma~\ref{lem:inexact_step} with $S_{K-1}(1)=K-1$. For (ii), $\eta hS_{K-1}(r)=1-r^{K-1}$ makes the prediction exact, and both factors of $\varphi$ lie in $[-1,1]$ when $|r|\le1$. The weight-prediction factor minus one is $\alpha u((K-1)u-K)$ with $u=\eta h$.
\end{proof}
All three updates move flat coordinates at the same speed; they differ on sharp ones. In our RL runs ($\tau=1$, $\alpha=0.5$, $K=5$, so $\gamma_{\rm hi}=3.5\eta$), PrePO is stable up to curvature $2/\eta$, one step of size $\gamma_{\rm hi}$ only up to $2/(3.5\eta)$, and weight prediction's step reaches magnitude $4$ at $2/\eta$. The two probes show which coordinates overshoot, and PrePO shortens exactly those.

The next result turns this per-coordinate step into a rate. It isolates one outer step on a fixed batch, where the batch loss is deterministic, and compares PrePO with one plain gradient step under the Polyak--{\L}ojasiewicz condition \citep{Karimi2016PL}.

\begin{proposition}[Per-batch progress]
\label{prop:pl_progress}
Let $f$ be $L$-smooth with minimum $f^\star$ and satisfy $\tfrac12\|\nabla f(x)\|^2\ge\mu\,(f(x)-f^\star)$. Take one PrePO step from $x$ with gradient motions ($C=I$, $A=-\eta I$), a diagonal trend clipped to $[-1,1]$, and $\eta_c=\tau\eta$ with $(1+2\tau)L\eta\le1$. Let $g=\nabla f(x)\ne0$, let $\Gamma=\alpha\eta(I+S)+\eta_cI$ with $S=S_{K-1}(R)$ be the per-coordinate step of Lemma~\ref{lem:inexact_step}, and let $\gamma_{\rm eff}=\langle g,\Gamma g\rangle/\|g\|^2\in[\gamma_{\rm lo},\gamma_{\rm hi}]$. Then the next point $x^+$ satisfies
\begin{equation}
 f(x^+)-f^\star\le\Bigl(1-2\mu\bigl(\gamma_{\rm eff}-c_KL\eta^2\bigr)\Bigr)\bigl(f(x)-f^\star\bigr),
 \label{eq:pl_progress}
\end{equation}
with $c_K=(\alpha K+\tau)(1+2\tau+2(\alpha K+\tau))$, while one gradient step of size $\eta$ gives the factor $1-2\mu(\eta-L\eta^2/2)$. Up to terms of order $L\eta^2$, one PrePO step therefore makes the progress of $\gamma_{\rm eff}/\eta$ plain steps, between $\alpha+\tau$ and $\alpha K+\tau$.
\end{proposition}
\begin{proof}
The probes give $\theta_1=x-\eta g$ and $g_1=\nabla f(\theta_1)$, and Eq.~(\ref{eq:prepo_final_predictor}) gives $\widehat\theta=\theta_1-\eta Sg_1$, so $\widetilde x=x-\alpha\eta(g+Sg_1)$ and $x^+=\widetilde x-\eta_c\nabla f(\widetilde x)=x-\Gamma g-e$ with $e=\alpha\eta S\delta_1+\tau\eta\delta_c$, $\delta_1=g_1-g$, and $\delta_c=\nabla f(\widetilde x)-g$. Smoothness gives $\|\delta_1\|\le L\eta\|g\|$, hence $\|g_1\|\le2\|g\|$, $\|\widetilde x-x\|\le\alpha\eta(2K-1)\|g\|$, and $\|e\|\le\alpha(1+2\tau)KL\eta^2\|g\|\le(1+2\tau)(\alpha K+\tau)L\eta^2\|g\|$, using $0\le S\le K-1$. The step condition then gives $\|x^+-x\|\le2\gamma_{\rm hi}\|g\|$. The descent lemma yields $f(x^+)\le f(x)-\gamma_{\rm eff}\|g\|^2+\|g\|\|e\|+2L\gamma_{\rm hi}^2\|g\|^2\le f(x)-(\gamma_{\rm eff}-c_KL\eta^2)\|g\|^2$, and the Polyak--{\L}ojasiewicz condition gives Eq.~(\ref{eq:pl_progress}). The gradient-step factor is the same argument with $\Gamma=\eta I$ and $e=0$.
\end{proof}
Coordinates that keep their pace contribute the full step $\gamma_{\rm hi}$ to $\gamma_{\rm eff}$, and coordinates that slow down contribute less, so the rate follows the measured trends. In our RL runs ($\alpha=0.5$, $\tau=1$, $K=5$) one PrePO step does the work of $1.5$ to $3.5$ plain steps on the batch loss, and Proposition~\ref{prop:not_larger_step} keeps that step stable where a single step of size $\gamma_{\rm hi}$ is not.

\begin{theorem}[Stochastic PrePO with AdamW-type state]
\label{thm:stochastic_prepo}
Under Assumption~\ref{ass:stochastic}, let $\kappa=\beta\gamma_{\rm lo}/(1-\beta)$ and $\varepsilon_z=L\kappa p_{\rm hi}c$, and suppose $64L\gamma_{\rm hi}^2p_{\rm hi}\le\gamma_{\rm lo}$. Then
\begin{equation}
 \frac1T\sum_{t<T}\mathbb E\|\nabla f(x_t)\|^2\le\frac{4[f(x_0)-f_{\inf}]}{p_{\rm lo}\gamma_{\rm lo}T}+A_\sigma\sigma^2+A_m\overline{\mathcal M}_T+A_b\bigl(\varepsilon_z^2+\overline{\mathcal B}_T\bigr)+A_cc^2+C_P\overline{\mathcal D}_T,
 \label{eq:stochastic_rate}
\end{equation}
where the bars average over $t<T$ the momentum error $\mathcal M_t=\mathbb E\|m_t-\nabla f(x_t)\|^2$, the drift $\mathcal B_t=\mathbb E[\|b_t\|^2+U_t^2]$, and the preconditioner change $\mathcal D_t=\mathbb E\|P_{t+1}-P_t\|$ (largest diagonal entry), with
\begin{gather*}
 A_\sigma=8\kappa_P\bigl(L(1+\varrho)p_{\rm hi}\gamma_{\rm hi}+\varrho(1-\beta^2)\bigr),\qquad A_m=8\varrho\beta^2\kappa_P,\qquad A_b=12(1+\varrho)\kappa_P,\\
 A_c=8L\varrho^2\beta^4\gamma_{\rm lo}\kappa_Pp_{\rm hi},\qquad C_P=\frac{4\beta c(G_f+2L\kappa p_{\rm hi}c)}{(1-\beta)p_{\rm lo}}.
\end{gather*}
The momentum error satisfies
\begin{equation}
 \overline{\mathcal M}_T\le\frac{3G_f^2}{(1-\beta^2)T}+3\Bigl(\frac{L\gamma_{\rm hi}p_{\rm hi}c}{1-\beta}+\bar b+\alpha Kac+\bar J\Bigr)^2+\frac{3(1-\beta)}{1+\beta}\sigma^2 .
 \label{eq:tracking}
\end{equation}
\end{theorem}
\begin{proof}
Track $z_t=x_t-\kappa P_tm_t$, for which the $m_t$ terms of Lemma~\ref{lem:inexact_step} cancel:
\[
 z_{t+1}-z_t=-P_t\bigl((\gamma_{\rm lo}I+F_t)G_t+E_tm_t+\gamma_{\rm lo}u_t\bigr)-\kappa(P_{t+1}-P_t)m_{t+1},\qquad\|u_t\|\le U_t.
\]
Apply the descent lemma to $f(z_{t+1})$, telescope over $t<T$, and return to $x_t$ with $\|\nabla f(z_t)-\nabla f(x_t)\|\le\varepsilon_z$; this gives Eq.~(\ref{eq:stochastic_rate}). The only nonstandard term is the noise times $P_tF_t$, which has nonzero mean because the trend is fitted on the same batch; it gives the $\varrho(1-\beta^2)$ part of $A_\sigma$. Eq.~(\ref{eq:tracking}) follows by unrolling $m_{t+1}-g_{t+1}=\beta(m_t-g_t)+(g_t-g_{t+1})+(1-\beta)(b_t+\delta_{c,t}+\xi_t)$.
\end{proof}
\paragraph{Gate switches.} If the gate switches $\alpha$ to zero once, apply the argument to each phase. The auxiliary point jumps by at most $|\kappa-\kappa'|p_{\rm hi}c$ at the switch, with $\kappa'=\eta_c\beta/(1-\beta)$, so Eq.~(\ref{eq:stochastic_rate}) holds with $f(x_0)-f_{\inf}$ increased by $G_f|\kappa-\kappa'|p_{\rm hi}c$ and $\gamma_{\rm lo}$ replaced by $\eta_c$ in the denominators.

\begin{corollary}[Rate and the price of prediction]
\label{cor:stochastic_rate}
Fix $K$, $\alpha$, $\tau$, and $\beta$, and set $\eta=\eta_0/\sqrt T$. For all sufficiently large $T$,
\[
 \min_{t<T}\mathbb E\|\nabla f(x_t)\|^2\le O\!\left(\tfrac1{\sqrt T}\right)+28\varrho(1-\beta)\kappa_P\sigma^2+O\bigl((1+\varrho)^3\kappa_P(\bar b^2+\bar J^2)\bigr)+C_P\overline{\mathcal D}_T.
\]
\end{corollary}
\begin{proof}
With $\eta\propto T^{-1/2}$, the step sizes, $\kappa$, $\varepsilon_z$, and $a$ are $O(T^{-1/2})$ while $\varrho$ stays fixed. Substituting Eq.~(\ref{eq:tracking}) into Eq.~(\ref{eq:stochastic_rate}), the $\sigma^2$ terms combine to at most $28\varrho(1-\beta)\kappa_P\sigma^2$ and everything else is $O(T^{-1/2})$, a bias or clipping term, or the preconditioner term.
\end{proof}
The baseline ($\alpha=0$) has $\varrho=0$ and the same bias and preconditioner terms, so prediction costs exactly two things. The first is a variance term, because the trend is fitted on the same batch as the gradient. It shrinks with batch size and would vanish for a trend fixed in advance, but such a trend could not see which coordinates overshoot. The second is the clipping term $\bar J$, which is zero whenever no token is clipped. The preconditioner term is zero for SGD with momentum, and at most $C_Pd(p_{\rm hi}-p_{\rm lo})/T$ for monotone denominators as in AMSGrad.

\paragraph{Weight decay.} Our runs use decoupled decay $\lambda=10^{-2}$. The endpoint results include it exactly through $C=(1-\eta\lambda)I$. In the convergence proof, decay adds a shrinkage $-\gamma_{\rm lo}P_t(\lambda_tP_t^{-1}x_t)$ with $\lambda_t\le(1+\varrho)\lambda$, plus $O(\eta)$ terms, so the theorem holds with $U_t$ increased by $(1+\varrho)\lambda\|(\sqrt{\widehat v_t}+\epsilon_{\rm opt})\odot x_t\|+O(\eta)$. At $\alpha=0$ this is AdamW's own decay term, which already moves AdamW's stationary points \citep{Zhuang2022AdamWProximal,Xie2024AdamWImplicit}; prediction scales it by at most $1+\varrho$.

\paragraph{Reading the constants.} Like other bounds for adaptive methods \citep{Chen2019AdamType,Defossez2022Adam}, Eq.~(\ref{eq:stochastic_rate}) is a worst case and does not predict the loss at LLM scale; $\kappa_P$ alone can reach $(c+\epsilon_{\rm opt})/\epsilon_{\rm opt}$. Its content is comparative: $L$, $\kappa_P$, $\sigma$, and $\bar b$ enter the $\alpha=0$ bound the same way, and the extra terms are either explicit in $(\alpha,K,\tau,\beta)$ or logged during training. Quantitative claims about the prediction rest on measured endpoint errors instead, which sit a factor of $1.2$ to $6$ below the bound (Table~\ref{tab:prediction-error-validation}).

\paragraph{Muon.} For matrix blocks $\Theta=(\Theta_b)$, let $\|X\|_{\rm op}=\max_b\|X_b\|_{\rm op}$, with dual norm $\|X\|_*=\sum_b\|X_b\|_*$, and $\bar r=\sum_b\min(m_b,n_b)$, so $\|X\|_*\le\sqrt{\bar r}\|X\|_F$. Muon keeps $M\leftarrow\mu M+(1-\mu)\bar g$ and moves each block by $-\eta\operatorname{polar}(M_b)$, the steepest-descent direction in $\|\cdot\|_{\rm op}$ \citep{Bernstein2024OldOptimizer}. Since $\|\operatorname{polar}(M)\|_{\rm op}\le1$ and $\langle M,\operatorname{polar}(M)\rangle=\|M\|_*$, for every $Y$
\begin{equation}
 \langle Y,\operatorname{polar}(M)\rangle\ge\|M\|_*-\|Y-M\|_*\ge\|Y\|_*-2\|Y-M\|_* .
 \label{eq:polar_alignment}
\end{equation}
PrePO predicts with the motions $U_j=-\eta\operatorname{polar}(M^{(j)})$ and a scalar trend $\rho_b\in[-1,1]$ per block, then corrects with $\Theta_{t+1}=\widetilde\Theta_t-\eta_c\operatorname{polar}(M_{t+1})$, where $M_{t+1}=\mu M_t+(1-\mu)\bar g_t(\widetilde\Theta_t)$.

\begin{theorem}[Stochastic PrePO with Muon]
\label{thm:muon_prepo}
Assume (i) $\|\nabla f(X)-\nabla f(Y)\|_*\le L_{\rm op}\|X-Y\|_{\rm op}$ and $f_{\inf}>-\infty$; (ii) $G_t=\bar g_t(\Theta_t)=\nabla f(\Theta_t)+b_t+\xi_t$ with $\mathbb E[\xi_t\mid\mathcal F_t]=0$, $\mathbb E\|\xi_t\|_F^2\le\sigma^2$, and $\|b_t\|_*\le\bar b$; (iii) $\|\bar g_t(\theta)-\bar g_t(\theta')\|_*\le L_B^{\rm op}\|\theta-\theta'\|_{\rm op}+J_t$ for the points used in step $t$, with $\mathbb EJ_t\le\bar J$; and (iv) a constant gate. Let $\gamma=\alpha K\eta+\eta_c$, $\gamma_{\rm lo}=\alpha\eta+\eta_c$, and $\varepsilon_0=(1-\mu)G_0-\nabla f(\Theta_0)$. Then
\begin{equation}
\begin{aligned}
 \frac1T\sum_{t<T}\mathbb E\|\nabla f(\Theta_t)\|_*\le{}&
 \frac{f(\Theta_0)-f_{\inf}}{\gamma_{\rm lo}T}+\frac{L_{\rm op}\gamma^2}{2\gamma_{\rm lo}}
 +\frac{\gamma}{\gamma_{\rm lo}}\Bigl[\frac{4\,\mathbb E\|\varepsilon_0\|_*}{(1-\mu)T}+\frac{4\mu L_{\rm op}\gamma}{1-\mu}\\
 &+6\sqrt{\bar r}\sigma\sqrt{1-\mu}+6\bar b+4(1+\alpha K)L_B^{\rm op}\eta+4\bar J\Bigr].
\end{aligned}
 \label{eq:muon_rate}
\end{equation}
With $1-\mu=T^{-1/2}$ and $\eta\propto T^{-3/4}$ the right side is $O(T^{-1/4})+(\gamma/\gamma_{\rm lo})(6\bar b+4\bar J)$, the rate of normalized SGD with momentum \citep{Cutkosky2020Momentum} and of Muon itself ($\alpha=0$).
\end{theorem}
\begin{proof}
Let $g_t=\nabla f(\Theta_t)$ and $\varepsilon_t=M^{(0)}-g_t$ for the first probe's moment. Every block moves by at most $\gamma$ with coefficient at least $\gamma_{\rm lo}$, so the descent inequality and Eq.~(\ref{eq:polar_alignment}) give
\[
 f(\Theta_{t+1})\le f(\Theta_t)-\gamma_{\rm lo}\|g_t\|_*+4\gamma\|\varepsilon_t\|_*+2\gamma(1-\mu)\bigl(\|b_t+\xi_t\|_*+\|\delta_1\|_*+\|\delta_c\|_*\bigr)+\tfrac{L_{\rm op}}2\gamma^2 .
\]
Bounding $\mathbb E\|\varepsilon_t\|_*$ from its momentum recursion and summing over $t<T$ gives Eq.~(\ref{eq:muon_rate}).
\end{proof}
For Muon the trend enters only through $\gamma/\gamma_{\rm lo}=1+\varrho$, because normalized motions keep every predicted step bounded, and the same-batch fit adds no separate noise term. Lemma~\ref{lem:clip_drift} gives (iii) with $L_B^{\rm op}=\bar rL_B$ and $\sqrt{\bar r}J_t$ in place of $J_t$. An approximate orthogonalizer such as Newton--Schulz, returning $O$ with $\|O\|_{\rm op}\le\omega_{\rm hi}$ and $\langle M,O\rangle\ge\omega_{\rm lo}\|M\|_*$, changes only the constants: $\gamma_{\rm lo}$ becomes $\omega_{\rm lo}\gamma_{\rm lo}$, the factor $2$ becomes $\omega_{\rm lo}+\omega_{\rm hi}$, and $\gamma$ becomes $\omega_{\rm hi}\gamma$ in the smoothness and drift terms.

\section{Additional Full-Parameter and Generality Results}
\label{app:additional}

\subsection{Primary Configurations}
\label{app:experimental-details}

Table~\ref{tab:primary-configurations} summarizes the main training and evaluation settings for
RLVR, SFT, and OPD. The PrePO settings $K$, $\alpha$, and $\eta_c$ were fixed before any evaluation, and none
was tuned on the reported benchmarks. Every reported score comes from the final checkpoint of its run.

\begin{table}[p]
\centering
\caption{Primary training and evaluation configurations.}
\label{tab:primary-configurations}
\small
\renewcommand{\arraystretch}{1.12}
\begin{tabular}{@{}>{\raggedright\arraybackslash}p{0.27\linewidth}>{\raggedright\arraybackslash}p{0.69\linewidth}@{}}
\toprule
Item & Setting \\
\midrule
\rowcolor{gray!15}
\multicolumn{2}{c}{\textit{RLVR}} \\
\midrule
Training pool & 17,917 DAPO-Math-17K prompts and 7,500 Hendrycks MATH training prompts (25,417 total), serialized in LightEval format~\citep{Yu2025DAPO,Hendrycks2021Math}. \\
Pool sampling & all\_probabilistic; $p_{\mathrm{easy}}{=}0.75$, $p_{\mathrm{hard}}{=}0.5$; refill each batch to its target size. \\
Models & Qwen2.5-Math-1.5B-Instruct, Qwen2.5-Math-7B-Instruct, Qwen2.5-3B-Instruct, and Llama-3.2-3B-Instruct~\citep{QwenMath2024Report,Qwen2024Report,Meta2024Llama32}. \\
Training stack & \textsc{verl} with FlashInfer~\citep{Sheng2025HybridFlow,Ye2025FlashInfer}. \\
Context and generation limits & 4096-position context; 1024-token prompt limit; 3,072-token response limit in training and evaluation. \\
Training sampler & Temperature 1.0; top-$p$ 1.0. \\
Learning rate & $10^{-6}$. \\
Optimizer and loss settings & AdamW; gradient clipping 1.0; policy clipping 0.2; entropy coefficient 0.001; no KL penalty. \\
Main-run budget & Train/minibatch sizes 256/64; $n{=}8$ responses per prompt; 400 steps; PrePO uses $K{=}5,\alpha{=}0.5$ for AdamW and Muon; entropy gate as in SFPO~\citep{Wang2026SFPO} with threshold 1 and patience 1. \\
Time to target & Target is GRPO's peak greedy Pass@1, averaged over the six benchmarks, over its checkpoints; steps and wall-clock time are recorded at the first checkpoint that reaches it. \\
Hardware and timing & $4\times$ H100 GPUs; wall-clock time is the elapsed time of full training steps, including rollout generation. \\
Muon sweep & Per arm, $K\in\{3,5,10\}$ and $\alpha\in\{0.1,0.3,0.5,0.8\}$. \\
RL ablation budget & Train/minibatch sizes 64/16; 300 steps; gate with threshold 1, patience 1, and warmup 30. \\
\midrule
\rowcolor{gray!15}
\multicolumn{2}{c}{\textit{SFT and OPD}} \\
\midrule
SFT dataset & Public K-and-K knights-and-knaves set~\citep{Xie2025Memorization}, converted to prompts, step-by-step reasoning, and canonical JSON answers. \\
SFT split & 3,600 training examples (900 each for 3--6 people); 400 held-out; 400 IID test; 200 OOD test examples for 7--8 people. \\
SFT models and budget & Qwen2.5-1.5B, Llama-3.2-3B, and Qwen2.5-7B; batch size 256; 400 steps. \\
SFT learning rate & $10^{-5}$. \\
OPD models and budget & Qwen2.5-1.5B, Qwen2.5-3B, and Qwen2.5-7B; batch size 256; $n{=}1$; 100 steps. \\
OPD objective & On-policy distillation~\citep{Agarwal2024GKD}: per-token reward $\log\pi_{\mathrm{T}}(y_t)-\log\pi_{\mathrm{old}}(y_t)$ on the student's own samples, centered by a top-$128$ student baseline and scaled by $0.1$; teacher Qwen2.5-Math-7B-Instruct for 1.5B/3B students, and Qwen2.5-32B-Instruct for the 7B student. \\
OPD learning rate & $5\times10^{-6}$. \\
SFT and OPD PrePO setting & Batch size 256; $K{=}5,\alpha{=}0.3$. \\
\midrule
\rowcolor{gray!15}
\multicolumn{2}{c}{\textit{Evaluation}} \\
\midrule
Benchmarks & MATH-500 (500), AIME24/25 (30 each), AMC23 (83), MinervaMATH (272), and OlympiadBench (674). \\
Decoding and metric & Greedy decoding (temperature 0), one response per problem; Pass@1 is computed for each benchmark. Avg.\ and Overall then average the six test columns. \\
Benchmark reporting & Scores at the final checkpoint of each run. \\
\bottomrule
\end{tabular}
\end{table}

\begin{table}[p]
\centering
\caption{Rollouts and generated tokens needed to reach the baseline target. Shaded ratios divide the
baseline's value by each method's value. $^\dagger$Scaled linearly from the step-100 count.}
\label{tab:token-efficiency}
\small
\setlength{\tabcolsep}{7pt}
\renewcommand{\arraystretch}{1.12}
\begin{tabular}{lrrr>{\columncolor{EffTint}}c>{\columncolor{EffTint}}c}
\toprule
Method & Steps $\downarrow$ & Generated tokens $\downarrow$ & Rollouts $\downarrow$ & Rollout ratio $\uparrow$ & Token ratio $\uparrow$ \\
\midrule
\rowcolor{gray!15}
\multicolumn{6}{c}{\textit{Qwen2.5-Math-1.5B, AdamW}} \\
GRPO & 392 & 738,973,428 & 802,816 & 1.00$\times$ & 1.00$\times$ \\
SFPO & 204 & 382,022,780 & 417,792 & 1.92$\times$ & 1.93$\times$ \\
\textbf{PrePO} & \textbf{113} & \textbf{212,310,988} & \textbf{231,424} & \textbf{3.47$\times$} & \textbf{3.48$\times$} \\
\midrule
\rowcolor{gray!15}
\multicolumn{6}{c}{\textit{Llama-3.2-3B, AdamW}} \\
GRPO & 285 & 378,267,892 & 583,680 & 1.00$\times$ & 1.00$\times$ \\
SFPO & 135 & 184,979,928 & 276,480 & 2.11$\times$ & 2.04$\times$ \\
\textbf{PrePO} & \textbf{90} & \textbf{128,451,916} & \textbf{184,320} & \textbf{3.17$\times$} & \textbf{2.94$\times$} \\
\midrule
\rowcolor{gray!15}
\multicolumn{6}{c}{\textit{Qwen2.5-Math-7B, AdamW}} \\
GRPO & 387 & 783,022,193 & 792,576 & 1.00$\times$ & 1.00$\times$ \\
SFPO & 107 & 242,768,854 & 219,136 & 3.62$\times$ & 3.23$\times$ \\
\textbf{PrePO} & \textbf{79} & \textbf{185,340,191} & \textbf{161,792} & \textbf{4.90$\times$} & \textbf{4.22$\times$} \\
\midrule
\rowcolor{gray!15}
\multicolumn{6}{c}{\textit{Qwen2.5-Math-1.5B, Muon}} \\
GRPO & 345 & 648,502,680 & 706,560 & 1.00$\times$ & 1.00$\times$ \\
\textbf{PrePO} & \textbf{115} & \textbf{224,500,002} & \textbf{235,520} & \textbf{3.00$\times$} & \textbf{2.89$\times$} \\
\midrule
\rowcolor{gray!15}
\multicolumn{6}{c}{\textit{Llama-3.2-3B, Muon}} \\
GRPO & 255 & 341,630,696 & 522,240 & 1.00$\times$ & 1.00$\times$ \\
\textbf{PrePO} & \textbf{90} & \textbf{128,389,664} & \textbf{184,320} & \textbf{2.83$\times$} & \textbf{2.66$\times$} \\
\midrule
\rowcolor{gray!15}
\multicolumn{6}{c}{\textit{Qwen2.5-Math-7B, Muon}} \\
GRPO & 340 & 688,272,384 & 696,320 & 1.00$\times$ & 1.00$\times$ \\
\textbf{PrePO} & \textbf{104} & \textbf{214,842,560} & \textbf{212,992} & \textbf{3.27$\times$} & \textbf{3.20$\times$} \\
\midrule
\rowcolor{gray!15}
\multicolumn{6}{c}{\textit{Qwen2.5-1.5B, OPD}} \\
OPD & 95 & 29,312,013 & 24,320 & 1.00$\times$ & 1.00$\times$ \\
\textbf{PrePO} & \textbf{20} & \textbf{5,607,440} & \textbf{5,120} & \textbf{4.75$\times$} & \textbf{5.23$\times$} \\
\midrule
\rowcolor{gray!15}
\multicolumn{6}{c}{\textit{Qwen2.5-3B, OPD}} \\
OPD & 65 & 19,795,125$^\dagger$ & 16,640 & 1.00$\times$ & 1.00$\times$ \\
\textbf{PrePO} & \textbf{34} & \textbf{10,042,332} & \textbf{8,704} & \textbf{1.91$\times$} & \textbf{1.97$\times$} \\
\midrule
\rowcolor{gray!15}
\multicolumn{6}{c}{\textit{Qwen2.5-7B, OPD}} \\
OPD & 52 & 16,240,640 & 13,312 & 1.00$\times$ & 1.00$\times$ \\
\textbf{PrePO} & \textbf{26} & \textbf{7,808,000} & \textbf{6,656} & \textbf{2.00$\times$} & \textbf{2.08$\times$} \\
\bottomrule
\end{tabular}
\end{table}

\begingroup
\setcounter{figure}{0}
\renewcommand{\thefigure}{A\arabic{figure}}
\renewcommand{\theHfigure}{app.\arabic{figure}}

\subsection{Optimizer and $K$--$\alpha$ Ablations}
\label{app:optimizer-k-alpha-ablation}

Figure~\ref{fig:gxpo-k-alpha-ablation} tracks accuracy over 300 steps for AdamW and Muon at
$K\in\{3,5\}$ and the $\alpha$ values shown in the legends. At $K=3$, $\alpha=0.5$ finishes
highest for both optimizers: AdamW improves steadily at this setting, while Muon's $\alpha=0.5$
curve finishes above its other plotted settings, including $\alpha=0.8$. At $K=5$, the ordering
depends on the optimizer. AdamW's $\alpha=0.5$ curve makes its strongest gains later in training
and finishes highest in its panel; Muon's $\alpha=0.3$ curve rises more steadily and finishes
highest, while its $\alpha=0.5$ and $\alpha=0.1$ curves remain lower and vary more over the
run. Comparing horizons, the $\alpha=0.3$ curves finish higher at $K=5$ than at $K=3$ for both
optimizers, whereas $\alpha=0.5$ finishes a little higher for AdamW and lower for Muon. This
pattern is consistent with an interaction between the optimizer, prediction horizon, and the
fraction of the predicted displacement that is applied. In PrePO, $K$ determines the virtual
horizon and hence the predicted endpoint $\hat\theta^{t,K}$, while $\alpha$ sets the fraction
of that displacement used to reposition the parameters,
$\tilde\theta^t-\theta^t=\alpha(\hat\theta^{t,K}-\theta^t)$. Thus, increasing $\alpha$
scales the move for a fixed endpoint, but changing $K$ changes the endpoint itself; the curves
show why the two choices should be read together. The strongest setting in this sweep varies by
optimizer and horizon, and the trajectories also show distinct paths to their final values, with
steady gains in some settings and late improvement or fluctuation in others.

\begin{figure}[t]
\centering
\includegraphics[width=\linewidth]{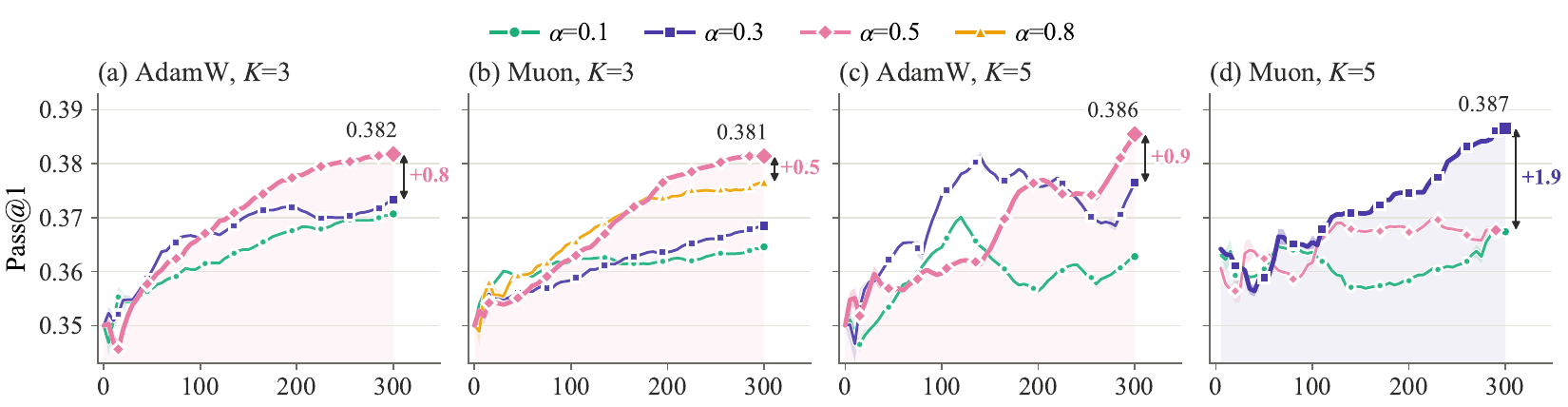}
\caption{Greedy Pass@1, averaged over the six benchmarks, over training steps for AdamW and Muon at $K=3$ and $K=5$. Each colour is one
$\alpha$. The shaded curve finishes highest in its panel; its final score is labelled, and the arrow gives
its lead over the next $\alpha$ in points.}
\label{fig:gxpo-k-alpha-ablation}
\end{figure}

\subsection{Looking Further Ahead Without Paying for It}
\label{app:horizon}

Two beliefs make multi-step methods look expensive and fragile. The first is that looking $K$ steps ahead
must cost about $K$ steps of compute. Figure~\ref{fig:horizon}(a) shows why it does not here. SFPO pays
$K{+}1$ backward passes per update, one for each real step and one for the correction, while PrePO pays
three whatever $K$ is, and the logged runs confirm the count at $K=5$ and at $K=10$, where a predicting
step costs 3.0 times a plain one. At $K=10$ that is 3.7 times fewer backward passes than SFPO. The second
belief is that the path of an LLM optimizer is too irregular to foresee from two steps. Running the real
$K$ steps next to the prediction on the same batch shows otherwise. The predicted move points the same way
as the real one, with a cosine of 0.993 for AdamW and 0.998 for Muon at $K=3$ that stays above 0.9 even at $K=10$, and it lands two to
five times closer to the real endpoint than stopping after the two probes. The remaining error grows with
$K$, which is why PrePO moves only a fraction $\alpha$ of the way and then takes one real corrective step,
and why the runs in Figure~\ref{fig:muon-alpha-signals} favour a smaller $\alpha$ at long horizons. The
path is smooth enough that repeating the last update lands in almost the same place, and that smoothness
is exactly what lets a two-step prediction see this far. Repetition matches PrePO here because the fitted
trends sit near one. The two differ on coordinates whose motion slows or reverses. On a quadratic,
repetition keeps the full step there and diverges at high curvature, while PrePO shortens it
(Proposition~\ref{prop:not_larger_step}).

\begin{figure}[t]
\centering
\includegraphics[width=\linewidth]{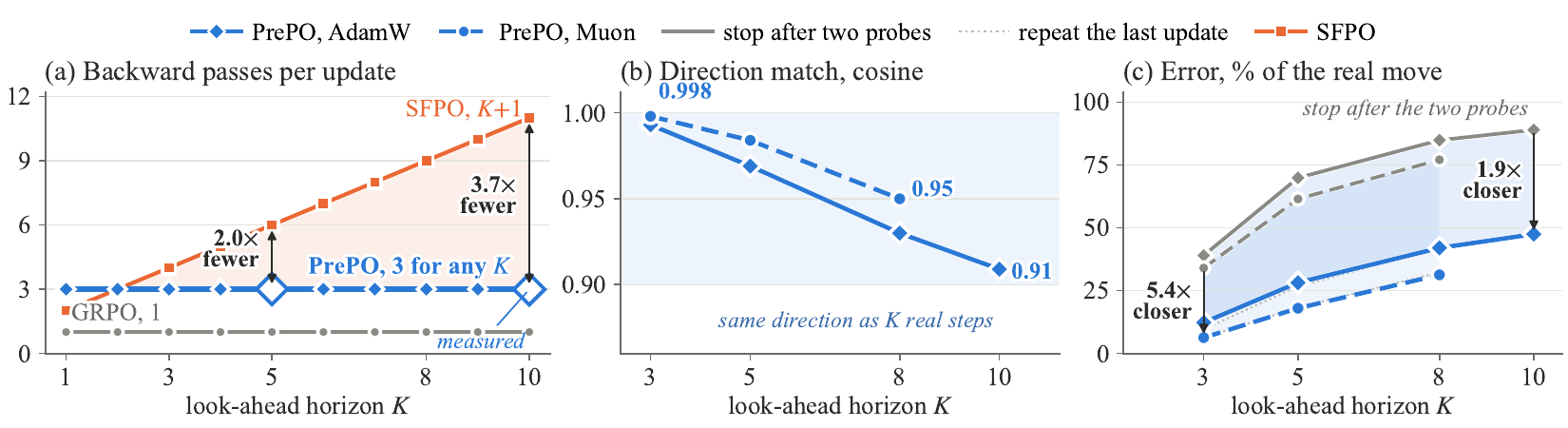}
\caption{(a) Backward passes per minibatch update against the look-ahead horizon $K$; open diamonds are
measured from the Muon $K=5$ and AdamW $K=10$ runs. (b, c) Qwen2.5-Math-1.5B with $\alpha=0.8$, with the real
$K$ steps run beside the prediction on the same batch: medians over training of the cosine between the predicted
and real moves, and of the endpoint error as a share of the real move.}
\label{fig:horizon}
\end{figure}

\subsection{What Prediction Costs and What \texorpdfstring{$\alpha$}{alpha} Changes}
\label{app:alpha-signals}

A training step splits its batch into minibatches and takes one optimizer update on each, so with 256
prompts per step and minibatches of 64 there are four updates per step. A plain update pays one backward
pass, SFPO pays $K$ or $K{+}1$, and PrePO pays three, two for the probes and one for the correction,
so at $K=5$ a step costs four backward passes for GRPO, twelve for PrePO, and twenty or more for SFPO.
Figure~\ref{fig:muon-alpha-signals} checks this count in Muon runs at $K=5$ and AdamW runs at $K=10$, which use 64 prompts
and minibatches of 16 and therefore also take four updates per step. Each minibatch is further split into
micro-batches by a token budget, so the logged count moves with response length, yet while PrePO is
predicting the count is always a multiple of three, and it falls from about 17.1 to 5.7 per step, a ratio
of 3.0, once the entropy gate hands training back to the plain optimizer. The gate closes between steps
82 and 88 for every $\alpha$, so prediction is used during the early phase where it saves the most steps,
and over 300 steps the runs use about half the backward passes they would if prediction stayed on
throughout. The value of $\alpha$ leaves this cost unchanged and instead sets how far each corrected update
lands from the current policy. A larger $\alpha$ raises the policy KL and the clip fraction while
prediction is on, and the clip fraction peaks at $2.1$, $3.2$, $3.6$, and $8.6\times10^{-4}$ for
$\alpha=0.1$, $0.3$, $0.5$, and $0.8$, so even the largest setting clips fewer than one token in a thousand.
Both signals return to the same level for every $\alpha$ once the gate closes. Response length follows the
same order and keeps it. All four runs start near 740 tokens, and over the last 50 steps they average
745, 746, 761, and 779 tokens, so a larger $\alpha$ moves the policy further along the direction in which
RLVR is already lengthening its reasoning, and that difference remains after prediction has stopped. The
AdamW runs at $K=10$ in the lower row repeat the cost story at twice the horizon. Doubling $K$ leaves every
update at three backward passes, 17.0 against 5.7 per step, and the gate again closes near step 87. What
changes is how strongly $\alpha$ matters. At this horizon the predicted endpoint lies further out, so the
same $\alpha$ moves the policy much more, and the clip fraction peaks at $2.9$, $20.9$, and
$100.4\times10^{-4}$ for $\alpha=0.1$, $0.3$, and $0.5$, about one token in a hundred at the largest
setting, before it falls back once prediction stops. Response length does not follow $\alpha$ here and ends
at 760, 795, and 742 tokens, so the lengthening seen with Muon depends on the optimizer and the horizon
rather than on $\alpha$ alone. Read together, the two rows favour a smaller $\alpha$ when the horizon is long.

\begin{figure}[t]
\centering
\includegraphics[width=\linewidth]{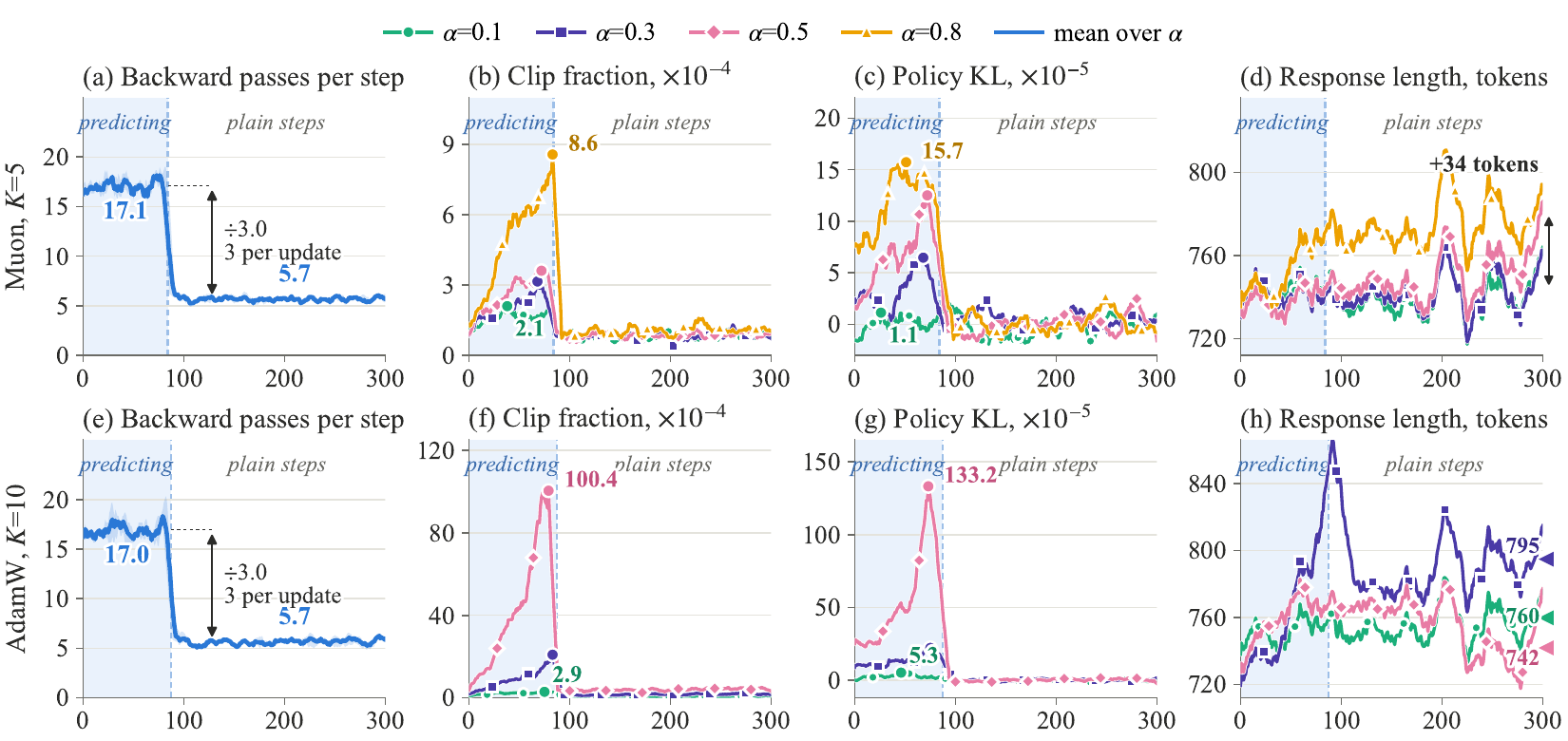}
\caption{Muon at $K=5$, top, and AdamW at $K=10$, bottom, over training steps. Shading marks where PrePO predicts
before the entropy gate turns it off; panels (a) and (e) show the mean over $\alpha$ with its range. Moving averages
span 9 steps, and 21 for KL and length.}
\label{fig:muon-alpha-signals}
\end{figure}

\endgroup

\subsection{Implementation Fidelity and Bound Validation}
\label{app:prediction-error-validation}

To check the implemented predictor against the finite-horizon analysis, Table~\ref{tab:prediction-error-validation}
compares the realized endpoint error, $\|\theta^{t,K}-\widehat{\theta}^{t,K}\|$, with the corresponding
theoretical bound across the listed models, optimizers, $K$, and $\alpha$ settings. The measured error
is below the reported bound in every row. Both generally increase with $K$: for Qwen2.5-Math-1.5B
with AdamW and $\alpha=0.8$, the error rises from $0.0029$ at $K=3$ to $0.0127$, $0.0363$, and
$0.0555$ at $K=5,8,10$, while the corresponding bounds rise from $0.0037$ to $0.0288$, $0.101$,
and $0.174$. Across the matched model, $K$, and $\alpha$ settings shown, Muon has lower measured
endpoint error than AdamW; at $K=5$ on the 1.5B model, reducing $\alpha$ from $0.8$ to $0.3$ also
lowers the measured error for both optimizers. Bounds marked $^\ast$ in the table are estimated, and the
rest are logged. Together,
these measurements provide an empirical check that the implemented predictions remain within the
reported bounds across the tested settings, consistent with the finite-horizon analysis in
Section~\ref{sec:theory}.

The settings in Table~\ref{tab:prediction-error-validation} are deliberately harder than the ones PrePO trains
with. The main runs use $K=5$ and $\alpha=0.5$, while the table reaches $K=10$ and $\alpha=0.8$, and even there
the bound stays above the measured error by a margin of $1.2$ to $6.0$, so the prediction never leaves the
region the analysis guarantees. The check also carries across model size. From the 1.5B to the 3B model at
$K=3$ and $\alpha=0.8$ the error grows, but the bound grows with it, and the error stays near four fifths of the
bound for both optimizers. The bound has no term in the number of
parameters; size enters only through the leftovers measured on the same model, which is why it keeps pace
as the model grows.

\begin{table}[H]
\centering
\caption{Endpoint error against the bound of Eq.~\ref{eq:main_prediction_bound}. Error is $\|\theta^{t,K}-\widehat\theta^{t,K}\|$, the distance between the predicted endpoint and the endpoint of $K$ real updates on the same batch, in the Euclidean norm (Frobenius for Muon matrices). Bound is the right side of Eq.~\ref{eq:main_prediction_bound} with $\varepsilon_e$ the largest measured leftover $\|e_j\|$. The margin is the bound over the error. $^\ast$Estimated rather than logged.}
\label{tab:prediction-error-validation}
\small
\setlength{\tabcolsep}{6pt}
\renewcommand{\arraystretch}{1.12}
\setlength{\fboxsep}{1.6pt}
\begin{tabular}{cc ccc ccc}
\toprule
 & & \multicolumn{3}{c}{\textbf{AdamW}} & \multicolumn{3}{c}{\textbf{Muon}} \\
\cmidrule(lr){3-5}\cmidrule(lr){6-8}
$K$ & $\alpha$ & Bound & Error & Margin & Bound & Error & Margin \\
\midrule
\rowcolor{gray!15}\multicolumn{8}{l}{\textit{Qwen2.5-Math-1.5B}} \\
3 & 0.8 & 0.0037 & \textbf{0.0029} & \mbadge{14}{1.3} & 0.0012 & \textbf{0.0010} & \mbadge{14}{1.2} \\
5 & 0.8 & 0.0288 & \textbf{0.0127} & \mbadge{19}{2.3} & 0.0066 & \textbf{0.0047} & \mbadge{15}{1.4} \\
5 & 0.3 & 0.0286 & \textbf{0.0048} & \mbadge{38}{6.0} & 0.0066 & \textbf{0.0018} & \mbadge{26}{3.7} \\
8 & 0.8 & 0.101$^\ast$ & \textbf{0.0363} & \mbadge{22}{2.8} & 0.023$^\ast$ & \textbf{0.0128} & \mbadge{17}{1.8} \\
10 & 0.8 & 0.174$^\ast$ & \textbf{0.0555} & \mbadge{24}{3.1} & \multicolumn{3}{c}{\textcolor{gray}{--}} \\
\rowcolor{gray!15}\multicolumn{8}{l}{\textit{Qwen2.5-3B}} \\
3 & 0.8 & 0.0058 & \textbf{0.0046} & \mbadge{14}{1.3} & 0.0022 & \textbf{0.0017} & \mbadge{14}{1.3} \\
5 & 0.5 & 0.042$^\ast$ & \textbf{0.0115} & \mbadge{26}{3.7} & 0.0105$^\ast$ & \textbf{0.0046} & \mbadge{19}{2.3} \\
\bottomrule
\end{tabular}
\end{table}

\subsection{Fallback Mechanism and Stability Ablation}
\label{app:fallback-ablation}

PrePO employs an entropy gate that triggers a fallback when output distribution entropy exhibits a sharp contraction. In our default configuration, this fallback is \emph{permanent}: once triggered, the repositioning coefficient $\alpha$ is locked to $0$ for the remainder of training, and the model updates solely via standard base optimizer steps. To evaluate this design choice, we compare three distinct fallback regimes on Qwen2.5-Math-1.5B with Muon over a 300-step training budget:
(1)~\textbf{No fallback}, where predictive repositioning remains active across all 300 steps regardless of entropy dynamics;
(2)~\textbf{Temporary fallback}, where predictive extrapolation is suspended for a single update ($\alpha=0$) whenever the entropy gate fires, but re-engages ($\alpha=0.3$) on subsequent steps if entropy stabilizes; and
(3)~\textbf{Permanent fallback (PrePO)}, where triggering the entropy gate permanently transitions training to standard optimizer updates ($\alpha=0$) for all remaining steps.

Table~\ref{tab:fallback-ablation} reports evaluation results across six mathematical reasoning benchmarks. Permanent fallback has the highest mean accuracy, $38.88$, ahead of temporary fallback at $37.70$ by $1.18$ points and no fallback at $31.07$ by $7.81$ points.

\begin{table}[H]
\centering
\caption{Ablation of the fallback mechanism on Qwen2.5-Math-1.5B with Muon over 300 steps. No fallback continuously extrapolates across all updates; Temporary fallback suppresses extrapolation for single steps upon entropy gate triggers; Permanent fallback (PrePO) permanently reverts to standard optimizer updates once the entropy gate fires.}
\label{tab:fallback-ablation}
\small
\setlength{\tabcolsep}{7pt}
\renewcommand{\arraystretch}{1.15}
\begin{tabular}{lccc}
\toprule
\textbf{Benchmark} & \textbf{No fallback} & \textbf{Temporary fallback} & \textbf{Permanent fallback (PrePO)} \\
\midrule
\textbf{Mean across benchmarks} & 31.07 & 37.70 & \textbf{38.88} \\
\midrule
MATH-500 & 70.60 & \textbf{77.20} & 75.40 \\
AIME 2024 & 3.33 & 10.00 & \textbf{13.33} \\
AIME 2025 & 10.00 & \textbf{16.67} & \textbf{16.67} \\
AMC 2023 & 42.17 & \textbf{56.63} & 55.42 \\
Minerva Math & 22.79 & 26.10 & \textbf{29.41} \\
OlympiadBench & 37.54 & 39.61 & \textbf{43.03} \\
\bottomrule
\end{tabular}
\end{table}

\textbf{Why the gate helps.} Early in training, up to about step 85, entropy is high and consecutive motions point the same way; the fitted Muon trend stays near $0.96$ (Figure~\ref{fig:muon-motion-entropy}a). Here the two-probe prediction is accurate and PrePO saves the most steps. Later, entropy falls sharply (Figure~\ref{fig:muon-motion-entropy}b) and the model settles into sharper regions, where the leftover term $\varepsilon_e$ of Eq.~\ref{eq:main_prediction_bound} is no longer small compared with the region it must stay in.

\textbf{No fallback.} Without the gate, PrePO keeps predicting through this late phase. Small prediction errors now push the repositioned point past the region the policy should stay in, the clipped surrogate discards more tokens, and accuracy declines by step 300. The mean falls to $31.07$, and AIME 2024 drops to $3.33$.

\textbf{Temporary fallback.} Pausing prediction for one step avoids this collapse but does not settle training. The pause steadies the policy, prediction restarts, the next predicted move overshoots, and the gate fires again, over 130 times in 300 steps. This on-off pattern keeps disturbing the optimizer's moment estimates. The mean reaches $37.70$ but trails permanent fallback on the hardest benchmarks, for example $39.61$ against $43.03$ on OlympiadBench.

\textbf{Permanent fallback.} Treating the first gate event as the end of the prediction phase keeps the benefit of both regimes. PrePO predicts while the path is regular, then the plain optimizer finishes without jumps in its state. It also costs less, because backward passes fall from $17.1$ to $5.7$ per step once prediction stops (Appendix~\ref{app:alpha-signals}). It has the best mean, $38.88$, and the best or tied-best score on four of the six benchmarks.

\subsection{Code Generation}
\label{app:code}

This run is a compact check that the gain holds outside math, alongside the logic puzzles of SFT and
the distillation runs, rather than a full code benchmark. We train Qwen2.5-Coder-1.5B \citep{Hui2024QwenCoder} with full-parameter RLVR on KodCode \citep{Xu2025KodCode} and evaluate on HumanEval+ \citep{Chen2021Codex,Liu2023EvalPlus}. GRPO and PrePO share a constant learning rate of $10^{-6}$ and no KL loss. PrePO uses $K=10$ and $\alpha=0.5$, with the remaining settings as in the main runs. We report greedy evaluation with Pass@1. The two runs end within one
point of each other, so we compare the time to reach $90\%$ of GRPO's final pass@1.

Table~\ref{tab:code} shows that PrePO leads at every checkpoint. The lead is largest early in training, $62.09$ against $54.47$ at step 10 and $66.82$ against $55.44$ at step 20. PrePO reaches the target with $3.4\times$ less wall-clock time and ends slightly higher, $71.06$ against $70.15$ at step 190. The gain therefore carries from math reasoning to code.

\begin{table}[H]
\centering
\caption{Code generation with Qwen2.5-Coder-1.5B trained on KodCode and evaluated on HumanEval+. Scores use greedy evaluation (Pass@1). The last column is the wall-clock speedup in reaching $90\%$ of GRPO's final pass@1.}
\label{tab:code}
\small
\setlength{\tabcolsep}{6pt}
\renewcommand{\arraystretch}{1.15}
\begin{tabular}{lcccccc}
\toprule
Method & Step 10 & Step 20 & Step 80 & Step 100 & Step 190 & Wall-clock speedup $\uparrow$ \\
\midrule
GRPO & 54.47 & 55.44 & 63.62 & 65.07 & 70.15 & 1.0$\times$ \\
\textbf{PrePO} & \textbf{62.09} & \textbf{66.82} & \textbf{68.98} & \textbf{70.91} & \textbf{71.06} & \textbf{3.4$\times$} \\
\bottomrule
\end{tabular}
\end{table}

\end{document}